\documentclass{article}

\usepackage[preprint]{neurips_2026}
\usepackage{verbatim}

\usepackage[utf8]{inputenc} % allow utf-8 input
\usepackage[T1]{fontenc}    % use 8-bit T1 fonts
\usepackage{hyperref}       % hyperlinks
\usepackage{url}            % simple URL typesetting
\usepackage{booktabs}       % professional-quality tables
\usepackage{amsfonts}       % blackboard math symbols
\usepackage{nicefrac}       % compact symbols for 1/2, etc.
\usepackage{microtype}      % microtypography
\usepackage{xcolor}         % colors

\usepackage{microtype}
\usepackage{graphicx}
\usepackage{subcaption}
\usepackage{booktabs} % for professional tables
\usepackage{amsmath}
\usepackage{amsthm}
\newtheorem{theorem}{Theorem}[section]
\newtheorem{definition}{Definition}[section]
\newtheorem{remark}{Remark}[section]
\newtheorem{lemma}{Lemma}[section]

\usepackage{amssymb}
\usepackage{mathtools}
\usepackage{wrapfig}
\usepackage{algorithm}
\usepackage{algorithmic}
\usepackage{fancyhdr}
\usepackage{listings}
\usepackage{eso-pic} % used by \AddToShipoutPicture
\usepackage{url}
\usepackage{caption}

\usepackage[export]{adjustbox}

\newcommand{\val}{\textsc{Val}}
\newcommand{\fpos}{f_{\textsc{Num}}}
\newcommand{\fsym}{f_ {\textsc{Let}}}

\newcommand{\spos}{s_\textsc{Pos}}
\newcommand{\ssym}
{s_\textsc{Sym}}

\newcommand{\Hcal}{H}

\newcommand{\posinp}{\sigma_{\textsc{Num}}}
\newcommand{\syminp}{\sigma_{\textsc{Let}}}

\newcommand{\softmax}{\mathrm{softmax}}
\definecolor{tabred}{rgb}{0.8, 0.18, 0.15}
\definecolor{tabblue}{rgb}{0.14, 0.46, 0.68}
\definecolor{tabgreen}{rgb}{0.14, 0.65, 0.15}
\definecolor{taborange}{rgb}{0.90,0.49,0.13} % a common orange tone

\newcommand{\mygreen}{\textcolor{tabgreen}{\textbf{green}}}
\newcommand{\myred}{\textcolor{tabred}{\textbf{red}}}

\lstdefinestyle{pythonstyle}{
  language=Python,
  basicstyle=\ttfamily\scriptsize,
  keywordstyle=\color{blue},
  commentstyle=\color{gray},
  stringstyle=\color{teal},
  numberstyle=\tiny,
  frame=none,
  breaklines=true
}

\title{Positional versus Symbolic Attention Heads: Learning Dynamics, RoPE Geometry, and Length Generalization}

\author{%
\small
\begin{tabular}{@{}c@{\hspace{4.0em}}c@{}}
\begin{tabular}[t]{c}
Felipe Urrutia \\
CENIA \& Faculty of Mathematics UC \\
Santiago, Chile \\
{\scriptsize\texttt{felipe.urrutia@uc.cl}}
\end{tabular}
&
\begin{tabular}[t]{c}
Juan José Alegría \\
CENIA \\
Santiago, Chile \\
{\scriptsize\texttt{jotaj.8a@gmail.com}}
\end{tabular}
\\
\noalign{\vskip 1.4em}
\begin{tabular}[t]{c}
Cinthia Sanchez Macias \\
CENIA \\
Santiago, Chile \\
{\scriptsize\texttt{mabel.sanchez.macias@gmail.com}}
\end{tabular}
&
\begin{tabular}[t]{c}
Jorge Salas \\
CENIA \\
Santiago, Chile \\
{\scriptsize\texttt{jorge.salas@cenia.cl}}
\end{tabular}
\\
\noalign{\vskip 1.4em}
\begin{tabular}[t]{c}
Cristian B. Calderon \\
CENIA \\
Santiago, Chile \\
{\scriptsize\texttt{cristian.buc@cenia.cl}}
\end{tabular}
&
\begin{tabular}[t]{c}
Cristobal Rojas \\
IMC UC \& CENIA \\
Santiago, Chile \\
{\scriptsize\texttt{luis.rojas@uc.cl}}
\end{tabular}
\end{tabular}
}

\begin{document}

\maketitle

\begin{abstract}
%Transformer-based language models increasingly shape high-impact domains, making it important to understand the mechanisms by which they solve structured tasks and to predict how they may behave in unseen scenarios.
%Mechanistic interpretability is crucial for explaining how Transformer-based models solve specific tasks and for anticipating how they may behave in unseen scenarios.

Transformer-based language models are widespread in today's society. As such, understanding the mechanisms by which they solve structured tasks and predicting how they may behave in novel scenarios is of great importance for safe deployment. We study the learning dynamics of attention heads in a controlled setting by training a decoder-only Transformer (GPT-J) on two structurally equivalent multi-hop reasoning tasks: a number task requiring positional reasoning and a letter task requiring symbolic reasoning. Using a recently introduced metric that classifies attention-head behavior as positional or symbolic for a given prompt, we show that successful learning is associated with the emergence of pure heads, i.e., heads that express themselves as either positional or symbolic. Despite the tasks' structural equivalence, they impose different mechanistic demands: the number task requires both positional and symbolic heads, whereas the letter task requires only symbolic heads.
We then identify the computational roles of these heads, characterize the basic functions they implement, and give theoretical constructions showing how single-layer RoPE-based attention can realize these functions through geometrically interpretable query, key, and value operations. This analysis yields a quantitative separation between positional and symbolic mechanisms in their robustness to longer sequences, formalized through a novel notion of discrepancy. We empirically validate the resulting predictions in both controlled and real-world models, showing that symbolic mechanisms extrapolate more reliably to longer sequences while positional mechanisms face sharper limitations.

\end{abstract}

\section{Introduction}

Transformer-based LLMs are shaping distinct aspects of society, ranging from education \citep{chu2025llm} to work productivity \citep{cambon2023early}. As such, it becomes pressing to develop a profound understanding of how these models actually work. Such an endeavor has led to the emergence of Mechanistic Interpretability as a scientific domain to systematically study how state-of-the-art Transformer-based (and other) architectures solve specific tasks ~\citep{conmy2023towards,elhage2021mathematical}. This field has had concrete impacts in bias control~\citep{hegde2024effectiveness}, AI safety~\citep{bereska2024mechanistic} and model customization~\cite{cammarata2025painting}, among other applications. 

A common line of inquiry has centered on elucidating the role of attention heads in a model’s underlying mechanisms~\citep{elhage2021mathematical} or circuits~\citep{lindsey2025landscape}. Recently, \citet{urrutia2025decoupling} proposed a novel interpretability method that can shed light on interesting properties of attention head behavior. In particular, this method reveals whether, given an input sequence, a particular head behaves positionally or symbolically on it. Intuitively, a positional head pays attention to specific locations in a given sequence, whereas a symbolic head pays attention to specific symbols, irrespective of their location (see \citet{urrutia2025decoupling} for details). Interestingly, using this method, these authors observed that in many real models, early layers heads are tuned to behave positionally, while deeper layers heads are tuned to behave symbolically. %Their analysis also reveals a clear association between different frequencies in RoPE-based models and head type behavior: while slow frequencies are used by the model to implement positional behavior, faster sequences support symbolic behavior. 

Understanding attention mechanisms through the lens of positional versus symbolic head behavior can be potentially relevant at the practical scale. Indeed, empirical evidence has shown that symbolic mechanisms are likely to have better generalization properties than positional ones~\citep{barbero2024round}. Moreover, positional and symbolic attention head behavior have been tightly linked to mechanisms underlying state-of-the-art Rotary Positional Encoding (RoPE; \cite{su2024roformer}). In particular, \citet{urrutia2025decoupling} uncovered an interpretable causal relation between RoPE frequency use, head behavior, and model performance of single layer models. 

Whereas most research has focused on investigating the behavior of trained models, evaluating the learning dynamics of the underlying attention mechanisms can provide interesting insights in terms of the relation between emergent model properties and its performance \citep{sharkey2025openproblemsmechanisticinterpretability}. For instance, examining the learning trajectories in a multi-hop variable-binding task allowed to reveal that early-formed heuristics persist in initial layers while systematic variable-tracing mechanisms gradually unfold in deeper layers \citep{wu2025transformers}. In this work we tread a similar road and study, in a controlled setting, how attention mechanisms unfold during training through the lens of positional versus symbolic attention heads. Given the connection between reasoning abilities and multi-hop tasks \citep{liao2021hop}, we do so within a multi-hop task framework (see Appendix \ref{sec:relatedwork} for a detailed description of related work). Our research consists of detailed empirical work\footnote{Code for our experiments can be found at \href{https://anonymous.4open.science/r/number-and-letter-neurips26-C754/}{\includegraphics[height=1em]{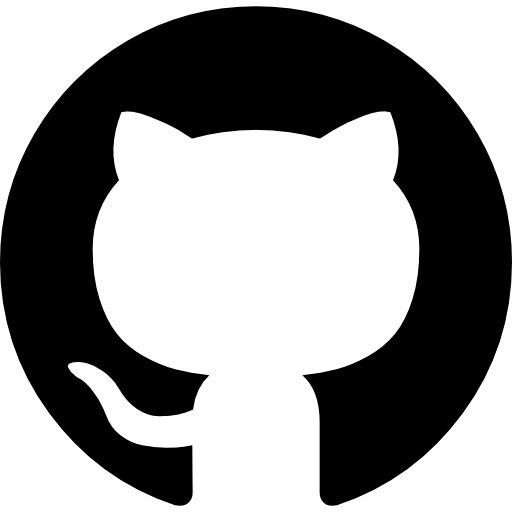} \texttt{number-and-letter-neurips26-C754}}.}, alongside theoretical analysis that allows us to make the following contributions:

\begin{enumerate}

  \item We present two structurally equivalent multi-hop reasoning tasks: the \emph{number task} that requires reasoning over positions, and the \emph{letter task} that requires reasoning over symbols.  We show that despite their structural equivalence, these tasks require distinct attention head behavior profiles. Moreover, by evaluating the learning profiles of positional and symbolic heads, we reveal that maximal accuracy in both tasks is related to the emergence of \textit{pure} attention heads, that is, heads expressing themselves as either positional or symbolic.

  \item By inspecting the information flow implemented by these pure heads, we uncover specific functions executed by each attention head-type to solve the tasks and provide a theoretical analysis of how these functions can, in principle, be implemented by single layers in isolation. We do so by presenting explicit constructions of query, key and value matrices within the standard RoPE arquitecture that realize each function through geometrically interpretable operations. We then use these geometric patterns to provide  evidence that the trained models largely implement the same mechanisms. 

  \item Our theoretical analysis reveals a strong  separation between positional and symbolic mechanisms in terms of their robustness to longer sequences. We introduce the notion of \emph{discrepancy} and use it to theoretically express this separation in a quantitative way.

  \item We empirically validate concrete predictions from our theoretical results, both in our controlled setting and in real-world models. In particular, we show that these models can solve the letter task on substantially longer input sequences than the number task. These findings provide strong evidence for RoPE’s distinct strengths and limitations in longer-context multi-hop reasoning.

\end{enumerate}

\section{Preliminaries}\label{sec:preliminaries}
\subsection{Number-based and letter-based multi-hop tasks}

\begin{figure*}[ht!]
    \centering
    \includegraphics[width=1\linewidth]{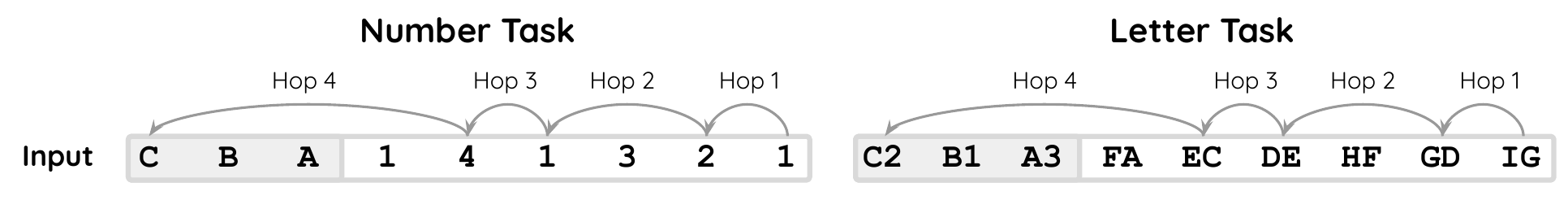}
    \caption{\textbf{Left: Number-based multi-hop task.} We illustrate this task via a 4-hop example, where numbers in the pre-target window (white background) define the steps to hop back until reaching the to-be-retrieved letter in the target window (grey background). %\textcolor{red}{\textbf{Red}} arrows identify the selective indexing mechanism to move the to-be-retrieved target towards the end of the sequence. \textcolor{gray}{\textbf{Grey}} letters depict the \textit{reflexive} mechanisms. 
    \textbf{Right: Letter-based multi-hop task.} Here, hops are defined by the letter pairs in the pre-target window, which determine the  steps to hop back until reaching the to-be-retrieved letter-number pair in the target window (grey background).% \textcolor{green}{\textbf{Green}} arrows identify the \textit{retrieval} mechanism to move the target towards the end of the sequence, and show how the symbol ``C2'' is dragged towards the end of the sequence. \textcolor{gray}{\textbf{Grey}} letters depict the \textit{reflexive} mechanisms.}
    \label{fig:main_hop}
    }
\end{figure*}

We evaluate two multi-hop tasks, one requiring reasoning over numbers and one requiring reasoning over letters (Figure \ref{fig:main_hop}). In general terms, a multi-hop task requires a sequence of interdependent reasoning steps (``hops''), where each step uses the output of the previous one. For both tasks, we define two input sequence windows over which hops can take place: a pre-target window (white background in Figure \ref{fig:main_hop}) and a target window (grey background in Figure \ref{fig:main_hop}). The pre-target window contains instructions to implement the multi-hop sequence, and the target window contains the response that needs to be retrieved. For both tasks, the multi-hop reasoning always starts from the end of the sequence, and finishes in some position of the target window. Figure \ref{fig:main_hop} depicts two examples of 4-hop instances in both tasks. Note that our dataset contains 1-, 2-, 3- and 4-hop sequences. For the sake of simplicity, our data set only contains sequences for which the multi-hop path is well defined and unique (we provide formal task descriptions together with intuitive dataset construction explanations in Appendix \ref{app:tasks}).

\textbf{Number task.} In the number-based task, numbers in the pre-target window define the hops to be executed. For instance, in Figure \ref{fig:main_hop} (left), the first number indicates that one must jump 1 step back on the sequence (hop 1), which leads to making another jump of 2 steps (hop 2), a third one of 1 step (hop 3), and a final hop of 4 steps (hop 4), to retrieve the letter ``C''.    

\textbf{Letter task.} In the letter-based task,  letter pairs in the pre-target window define the hops to be executed. More specifically, the second letter in the pair defines which first letter of an earlier pair in the sequence the hop must reach. For instance, in Figure \ref{fig:main_hop} (right), the second letter of the first pair (``G'') indicates that hop 1 must step to the pair with ``G'' as its first letter. In turn, the second letter of that pair (``D'') indicates that hop 2 must step back to the pair with ``D'' as its first letter, and so on until retrieving a number-letter pair ``C2'' in the target window.

Our tasks are chosen to provide a simple controlled setting for mechanistic analysis. Although structurally similar, they are expected to require distinct mechanisms, whose analysis may help clarify how more complex tasks are solved through their combination.

\subsection{Positional and symbolic attention head scores}

{To assess mechanistic differences between the tasks, we use the scores introduced by \citet{urrutia2025decoupling}, which characterize whether an attention head behaves positionally or symbolically. Informally, a head is positional if its attention distribution is invariant to input-token permutations, and symbolic if its attention distribution is equivariant to such permutations; see Appendix \ref{app:scores} for formal details and motivations for the use of these metrics.}

\section{Results}\label{sec:emergence}

To develop mechanistic intuition and enable interpretable task comparisons, we train a standard 12-layer RoPE Transformer with one attention head per layer on both tasks (see Appendix \ref{app:architecture} for architecture, training, and parameter details). This single-head controlled setting reduces within-layer interactions, allowing us to isolate positional and symbolic mechanisms while preserving the core RoPE-based attention computation. We later test the resulting predictions in real-world multi-head models, supporting the relevance of this simplified setup.

\subsection{Relation between attention head behavior and task accuracy}\label{sec:learning_dynamics}

\begin{figure*}[!ht]
\centering
\includegraphics[width=1\textwidth]{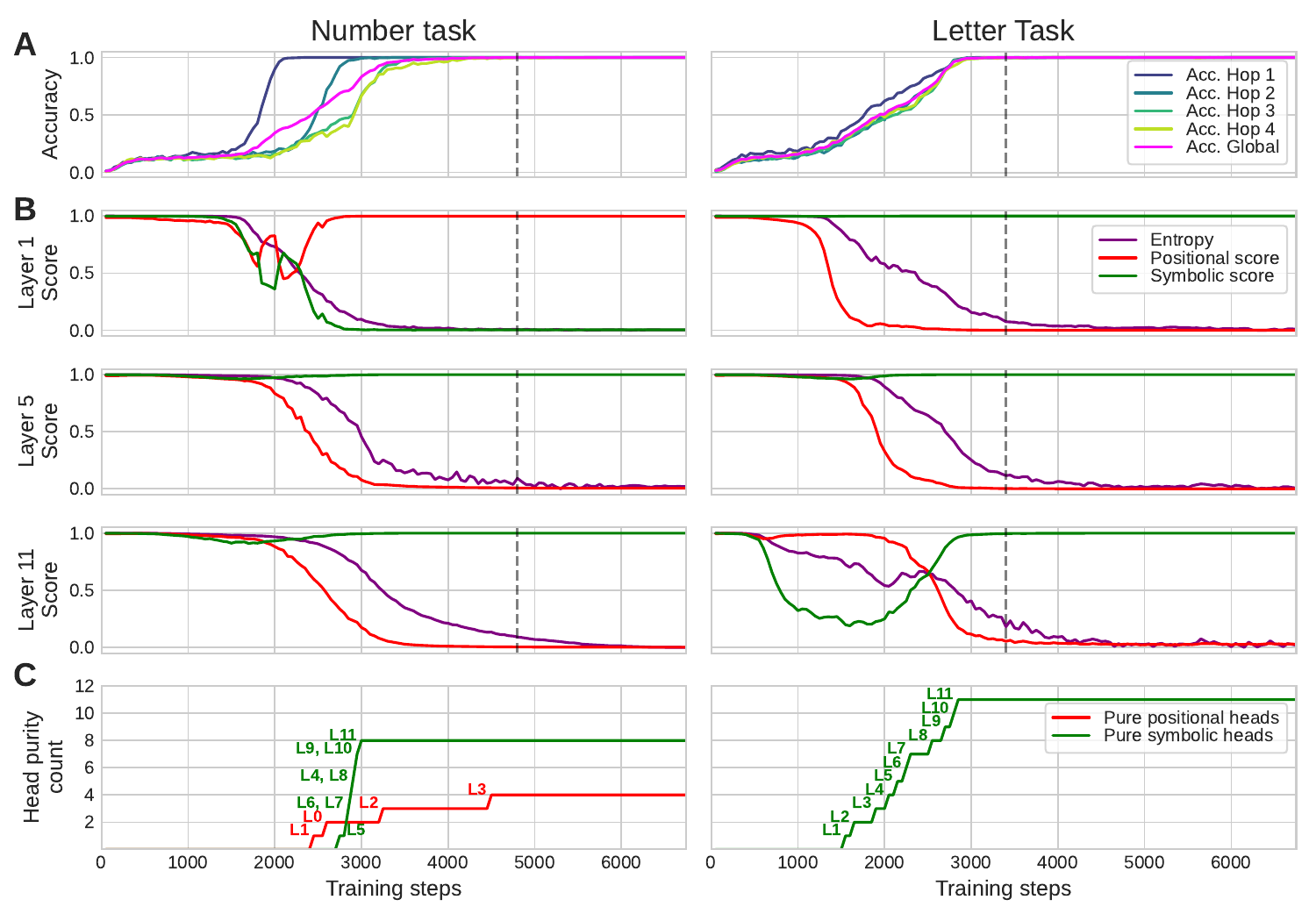}
    \caption{\textbf{A. Task accuracy.} Left: Lines depict the test set accuracies on the number task for all $n$-hop conditions as a function of training stems (color code in the legend), and averaged across all conditions (magenta line). Right: Same as left but for the letter task. Here, and in graphs from B, dashed vertical lines depict the moment of maximal task accuracy performance. \textbf{B. Positional and symbolic attention head score dynamics.} Columns represent tasks (left = number task, right = letter task) and rows represent layer depths (1, 5 and 11, see Appendix \ref{app:full_results} for results in all layers). In all graphs, red and green lines reflect the average positional and symbolic scores as a function of training steps. Purple lines depict the average normalized entropy values of the attention distribution. \textbf{C. Dynamics of head purity counts.} Red and green lines reflect the temporal tally of head purity for positional and symbolic heads, respectively (left = number task, right = letter task). We further indicate the layer depth number associated to the emergence of head purity.} \label{fig:acc}
\end{figure*}

We evaluated whether average (across hop conditions) task accuracy (magenta line in Figure \ref{fig:acc}A) was associated with the emergence of attention head behavior. The learning dynamics of the latter are reflected by red and green lines in Figure \ref{fig:acc}B, respectively for positional and symbolic attention head scores of layers 1, 5 and 11 (see Appendix \ref{app:full_results} for results on all layers). Note that both scores always start high, which according to the exclusion principle \citep{urrutia2025decoupling}, is indicative of a uniform distribution over attention scores, here reflected by high attention distribution entropy values (purple line in Figure \ref{fig:acc}B)\footnote{This is simply due to the fact that weights are initialized randomly w.r.t. a gaussian distribution centered at 0.}. After some learning steps, interesting patterns emerge. In both tasks, we observed that heads behavior must be \textit{pure}\footnote{We consider a head as pure once positional and symbolic scores have both converged to sufficiently high and low values} in order for test accuracy to converge to its maximal value. We temporally aligned the moment at which accuracy converges to its maximal value (maximal accuracy) with head behavior dynamics, and depicted the time point of maximal accuracy (grey dashed lines in Figure \ref{fig:acc}A) in all subplots of Figure \ref{fig:acc}B. We observe that when accuracy converges, both head types are pure. The later implies that the underlying mechanisms giving rise to this purity are tightly related to task learning. Such results validate the utility of these metrics, which act as a powerful lens tool that can purposefully direct, more fine-grained mechanistic assessments responsible of solving the tasks (see Section \ref{sec:mechanisms}).

\subsection{Mechanism implemented by pure heads}
\label{sec:mechanisms}

To move from head-purity scores to concrete computations, we inspected the attention patterns and information flow of trained models on correctly solved inputs. This analysis reveals that pure heads are mainly implementing a small set of recurring operations, which naturally decompose the two multi-hop tasks into simpler sub-tasks.

\begin{figure*}[ht!]
    \centering
    \begin{subfigure}{0.47\textwidth}
        \centering
        \includegraphics[width=1.05\linewidth]{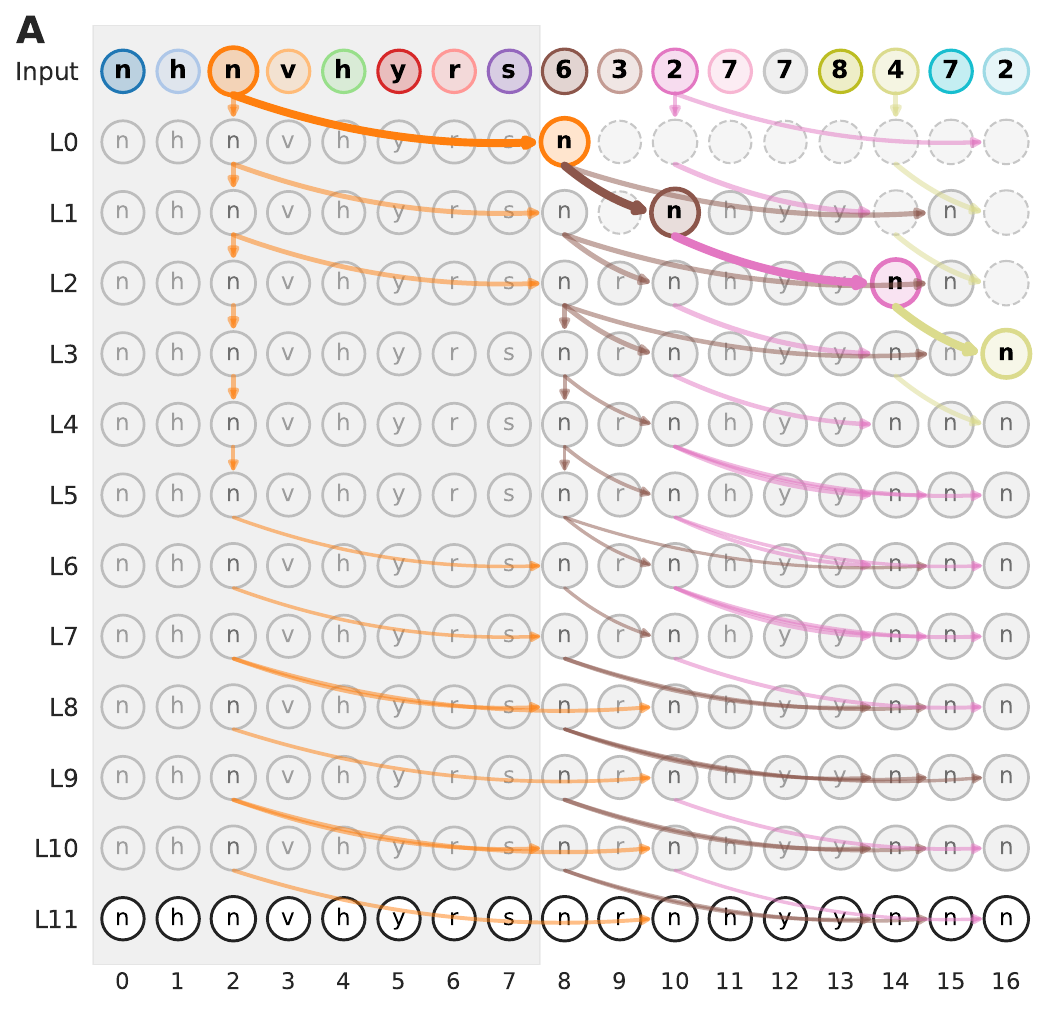}
    \end{subfigure}
    \hfill
    \begin{subfigure}{0.47\textwidth}
        \centering
        \includegraphics[width=1.05\linewidth]{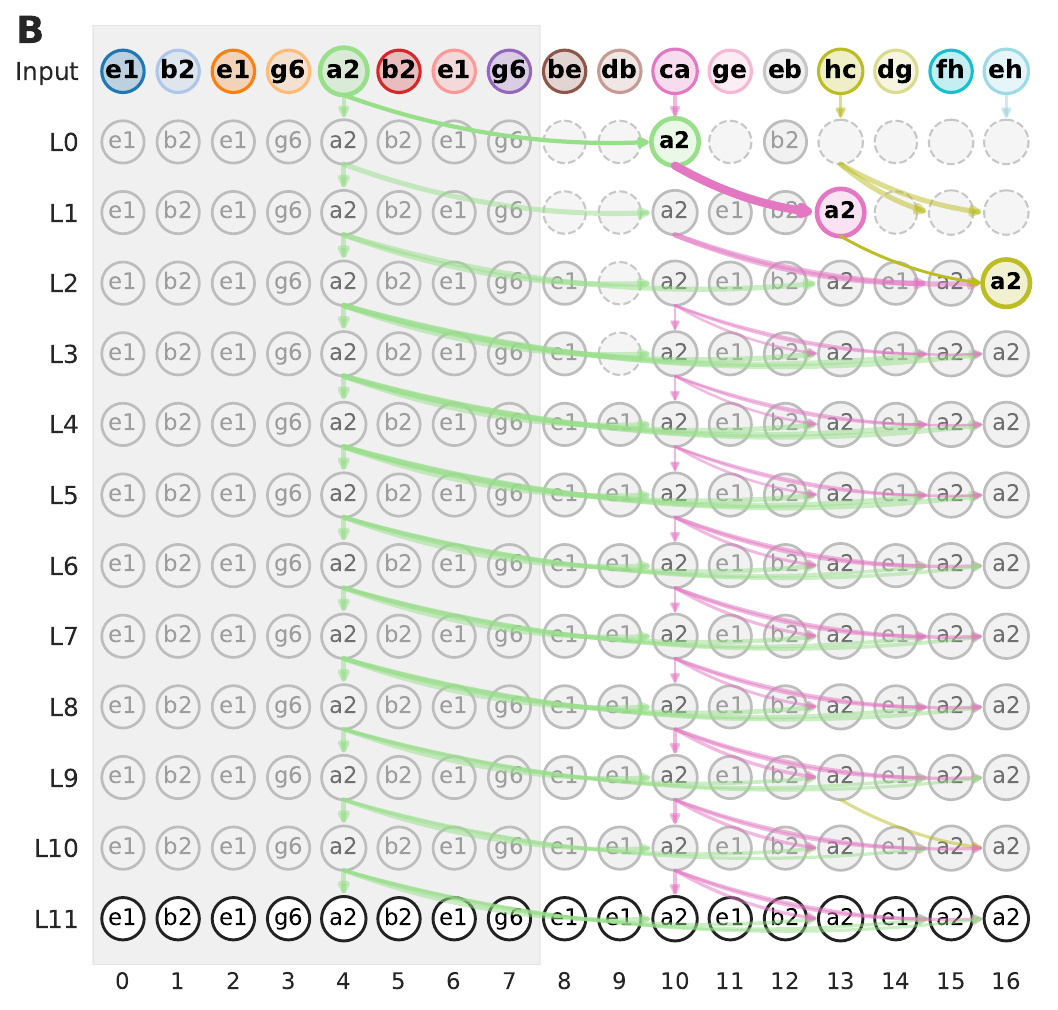}
    \end{subfigure}
    \includegraphics[width=1\linewidth]{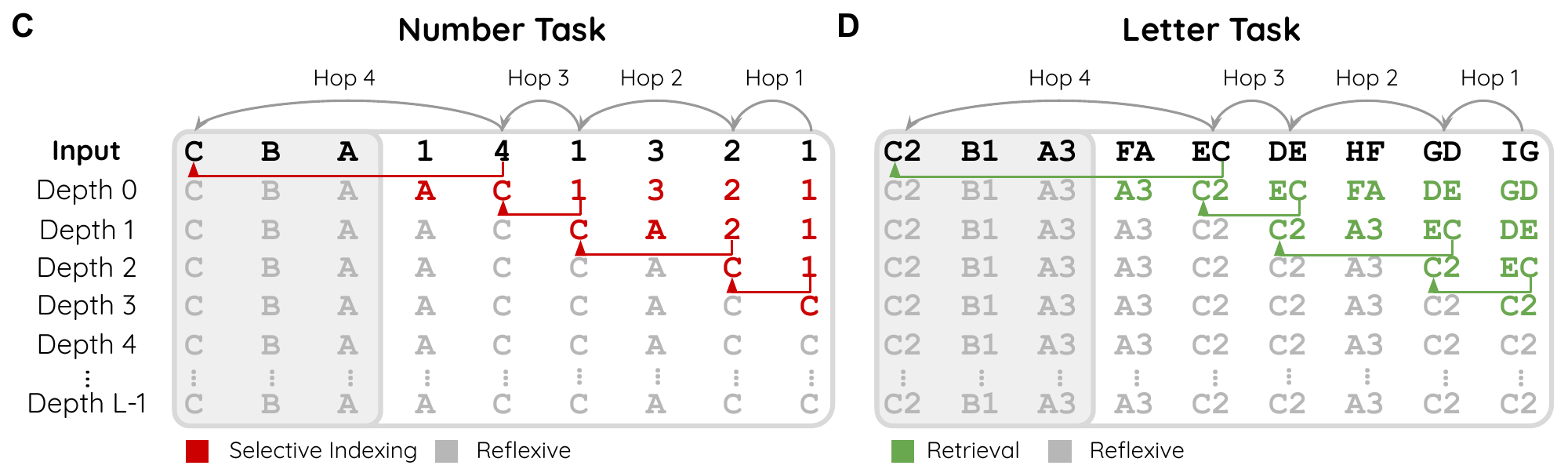}
    \caption{\textbf{Left Column: Number task.} We illustrate this task via a 4-hop example, where numbers in the pre-target window (white background) define the steps to hop back until reaching the to-be-retrieved letter in the target window (grey background).  
    \textbf{Right Column: Letter  task.} Here, hops are defined by the letter pairs in the pre-target window, which determine the  steps to hop back until reaching the to-be-retrieved letter-number pair in the target window (grey background). \textbf{Upper Row:} Trained GPT-J model. \textbf{Lower Row:} Idealized Index, Retrieval and Reflexive functions.}
    \label{fig:pos_sym_analysis}
\end{figure*}

Figures \ref{fig:pos_sym_analysis}\textbf{A} and  \ref{fig:pos_sym_analysis}\textbf{B} illustrate this flow for the Number and the Letter tasks, respectively. In the Number Task for example, if we focus on the last hop of the input sequence (from the token "6" to token "n") and evaluate the attention distribution from token "6" (the query), we observe that layer 0 maximally attends to "n", and given that this token is a letter, copies it. If we now fall back to the previous hop (from "2" to "6"), we observe that layer 1 again maximally attends the position of token "6", and given that it represents a letter ("n" in this case), copies it again. On the other hand, if we now focus on the last token of the input sequence (token "2"), we observe that in layers 1 and 2, attention is maximized at the token located at distance 2 from the last token (as it should be, since the last token holds a "2"), but there is no information flow. Only at layer 3, once the letter "n" has flown to that location, it is copied to the last position. To abstract away this operation, we introduce an idealized version that we call the \emph{Selective Index} function. 

From there on, since the answer is already at the the last token, information only needs to be propagated throughout the last layer. We call the corresponding idealized operation the \emph{Reflexive} function. An analogous analysis can be extrapolated to the letter task. We refer to the resulting idealized core operation as the \emph{Retrieval} function.

Selective Indexing is positional. At a position containing a number, it uses that number as a relative offset and copies the token found at the corresponding earlier position, but \emph{only} when the retrieved token is a letter. This is the core operation used in the number task. The Retrieval function is symbolic. If the current token is a letter-letter, it follows a symbol-to-symbol association by using the current token to identify the earlier pair containing the matching symbol, and always copies that pair. Finally, the Reflexive function simply preserves and propagates the current token representation through later layers. See Listing \ref{listing:functions} for pseudo-code formalization of these basic functions and Appendix \ref{sec:functions-pseudocode} for a formal definition. 

\begin{wrapfigure}{r}{0.45\textwidth}
\vspace{-1.8em}
\begin{lstlisting}[
  caption={Pseudocode of Index, Retrieval, and Reflexive functions.},
  captionpos=b,
  label={listing:functions},
  language=Python,
  basicstyle=\ttfamily\scriptsize,
  keywordstyle=\color{blue},
  commentstyle=\color{gray},
  stringstyle=\color{teal},
  numberstyle=\tiny,
  frame=single,
  breaklines=true
]
def index(seq):
    out = seq.copy()
    for n in range(len(seq)):
        if is_number(seq[n]):
            i = seq[n]
            if is_letter(seq[n-i]):
                out[n] = seq[n-i]
    return out

def retrieval(seq):
    out = seq.copy()
    for n in range(len(seq)):
        if is_letter(seq[n]):
            for i in range(n):
                if seq[i][0] == seq[n][1]:
                    out[n] = seq[i]
                    break
    return out

def reflexive(seq): return seq
\end{lstlisting}
\vspace{-2.5em}
\end{wrapfigure}

As we will show in later sections, these functions are not merely descriptive abstractions. We prove that \emph{Selective Indexing} and \emph{Retrieval} can each be realized by a single RoPE-based attention layer in isolation, and that their required mechanisms are fundamentally different: \emph{Selective Indexing} (from here on termed \emph{Index}) must be implemented positionally, whereas \emph{Retrieval} must be implemented symbolically. However, \emph{Reflexive} can in principle be implemented by either positional or symbolic mechanisms. The fact that trained models implement it symbolically is therefore a nontrivial empirical observation, made visible by the positional-symbolic scores.

These functions give a compact mechanistic decomposition of both tasks. For an input $\posinp$ to the number task with at most $h$ hops and a model with $\ell \geq h$ layers, the computation can be written as repeated applications of \emph{Index}, followed by \emph{Reflexive} propagation (see Figures \ref{fig:pos_sym_analysis}C and \ref{fig:pos_sym_analysis}D for an illustration):
\[
\fpos(\posinp) = \textcolor{black}{\mathrm{Reflexive}}^{\ell-h}\bigl(\textcolor{black}{\mathrm{Index}}^h(\posinp)\bigr)_n.
\]
Analogously, for an input $\sigma_{Let}$ to the letter task with length $n$ and $\leq h$ hops we have:
\[
\fsym(\syminp) = \textcolor{black}{\mathrm{Reflexive}}^{\ell-h}\bigl(\textcolor{black}{\mathrm{Retrieval}}^h(\syminp)\bigr)_n.
\]

Thus, although the two tasks have the same multi-hop structure, they rely on different core mechanisms: the number task is solved through repeated positional indexing, whereas the letter task is solved through repeated symbolic retrieval. In both cases, later layers can then preserve the retrieved information through reflexive propagation. 

Such observations provide the basis for our theoretical analysis of how these functions can be implemented by single-layer RoPE-based attention.

\subsection{Mechanistic explanation of hop-wise learning dynamics}

We next use this mechanistic decomposition to explain the different $n$-hop learning dynamics observed in Figure~\ref{fig:acc}A. In the number task, accuracy emerges progressively: the model first solves 1-hop inputs and then gradually learns 2-, 3-, and 4-hop conditions. This pattern follows from the mechanism above: reliably solving an $n$-hop condition requires the emergence of sufficiently many positional heads to perform successive selective-indexing steps, together with symbolic heads that propagate the retrieved information to the final layer. Consistent with this explanation, the cumulative emergence of pure heads shows a step-like pattern for positional heads, while symbolic propagation is already available early enough to support later positional computations (Figure~\ref{fig:acc}C).
In contrast, the letter task does not exhibit a hop-wise learning gradient. Since the task is solved primarily through symbolic heads, all hop conditions depend on the emergence of the symbolic computation in the final layer. The last required symbolic head appears in layer 11, which is needed for every $n$-hop condition; consequently, the model acquires the different hop conditions nearly simultaneously rather than progressively. Thus, the contrast between the tasks is not explained by hop count (i.e., intrinsic task difficulty) alone, but by the different temporal emergence of the mechanisms required to solve them.

\subsection{Plausible Geometric Mechanisms for realizing the Selective Index and Retrieval functions}\label{sec:feasibility}

We show that the Index and Retrieval functions can both be implemented by single-layer decoder-only Transformers with a single attention head and no activation function, ensuring that the relevant computation is carried primarily by the attention module. Importantly, these constructions do not rely on ad-hoc components: they use standard RoPE with a single frequency and softmax attention. Our goal is not only to establish expressivity in the usual sense, but to provide geometrically interpretable mechanisms that are plausible to learn and whose signatures can be checked in trained models. This connection is what later allows us to relate the theoretical limitations of the constructions to the behavior of trained controlled models and real-world models.

\begin{theorem}[Feasibility]\label{teo:mechanisms}
Fix a maximum input context-length $n_0> 1$. Then, there exists a Transformer $T_\mathrm{Index}$ with a two-dimensional RoPE-based attention module with angle $\theta$, such that for every input $\bar{\sigma}$ (valid for the multi-hop number task) with length $n \leq n_0$, $T_\mathrm{Index}(\bar{\sigma}) = \mathrm{Index}(\bar{\sigma})$. Furthermore, there exists a Transformer $T_\mathrm{Retrieval}$ with a two-dimensional  RoPE-based attention module with angle $\theta$, such that for every input $\bar{\sigma}$ (valid for the multi-hop letter task) with length $n \leq n_0$, $T_\mathrm{Retrieval}(\bar{\sigma}) =\mathrm{Retrieval}(\bar{\sigma})$.
\end{theorem}

\begin{wrapfigure}{r}{0.6\textwidth}
\vspace{-2em}
    \centering
\includegraphics[width=0.95\linewidth,valign=c]{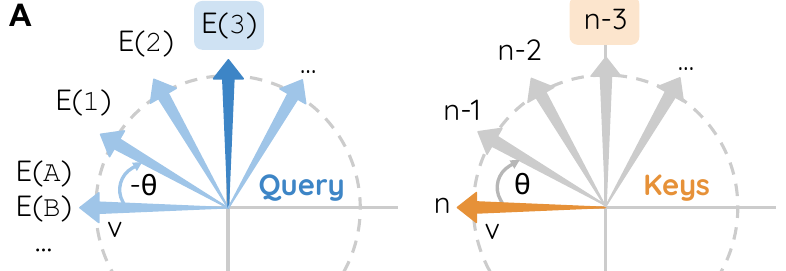} %\includegraphics[width=1\linewidth,valign=c]{figures/positional_tie_rope_flow_layer1_plane1_bix4.pdf}
\includegraphics[width=0.95\linewidth,valign=c]{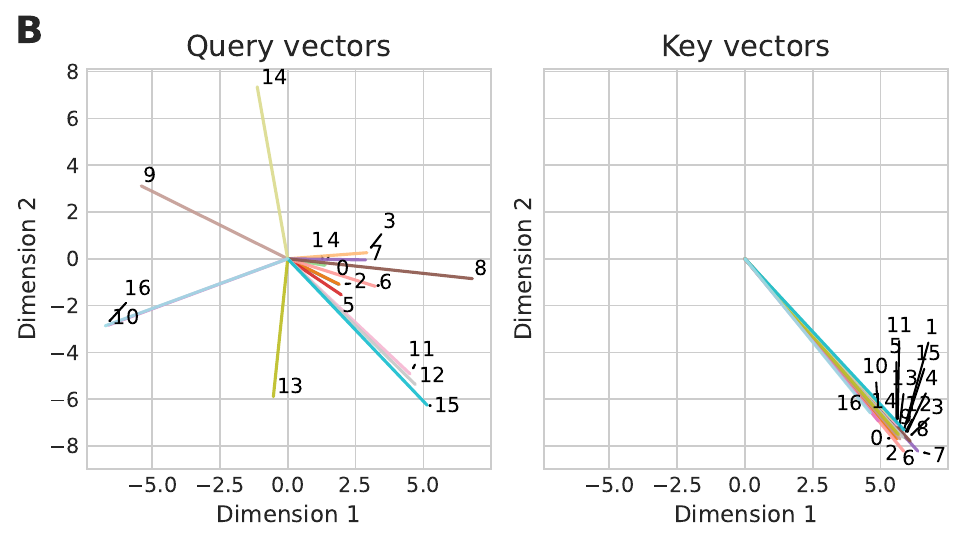}
\includegraphics[width=0.95\linewidth,valign=c]{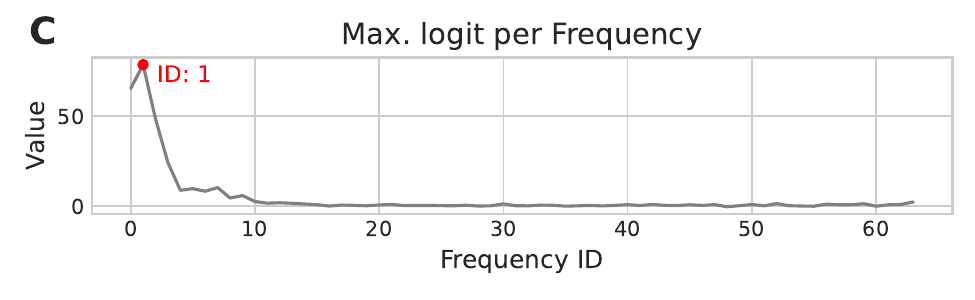}
    \caption{Illustration of the Index mechanism in the theoretical construction (\textbf{A}) and in a trained model (\textbf{B}). Panel \textbf{C} identifies the dominant RoPE frequency plane at layer L1 for the input in Figure~\ref{fig:pos_sym_analysis}.}
    \label{fig:ill_index_task}
\vspace{-2em}
\end{wrapfigure}

The proof, that can be found in Appendix \ref{proof:teo}, relies on simple and interpretable geometric mechanisms. Figure \ref{fig:ill_index_task}A illustrates the mechanism used to implement the Selective Index function.

The idea is very simple. The key point is that, in order for attention weights to depend only on position, all the key vectors must be aligned, for example equal to some arbitrarily chosen vector, a rol played by $v$ in Figure \ref{fig:ill_index_task}A. Then, different query vectors must hard-code the different query positions (where to look) by organizing in such a way that after RoPE, each of them has been rotated exactly to match the vector $v$ at the queried position, and not elsewhere. Figure \ref{fig:ill_index_task}B shows the key and query vectors in layer 1 of our trained model, projected to the two dimensional rotational plane corresponding to rotation angle in RoPE that contributes the most to the attention distribution (see Figure \ref{fig:ill_index_task}C). Note the similarity of the keys and queries organization pattern between the model and the theoretical construction.    

\medskip 

\begin{remark}\label{rem:impossibility}
   Note that our constructed models $T_\mathrm{Index}$ and $T_\mathrm{Retrieval}$ have, respectively, purely positional and purely symbolic attention heads. In fact, as we show in Appendix \ref{app:proof_rem}, the Index function cannot be realized by a purely symbolic head, while the Retrieval function cannot be computed by a purely positional one. 
\end{remark}

\subsection{Robustness Analysis}

 A fundamental limitation shared by all such constructions arises from the use of softmax attention with bounded logits. As the input length increases, attention weights necessarily disperse across more tokens, reducing the effective signal carried by any single attended position \citep{pasten2025continuity}. This phenomenon imposes an intrinsic obstacle to length generalization that applies to both mechanisms considered here. To study more fine-grained differences between mechanisms, we introduce a quantity that captures limitations of the logit function to focus on a suited token, independently of the softmax normalization. 

\begin{definition}[Discrepancy] The \emph{discrepancy} $\Delta$ of an attention head is the maximal difference between the largest and second-largest logit produced by the model on inputs of a certain length. 
\end{definition}
Intuitively, a discrepancy of zero means that the model will not be able to distinguish the correct token, potentially leading to incorrect responses. The following theorem provides upper bounds for the discrepancies of our constructed models. 

\begin{theorem}[Robustness]\label{teo:discrepancy}
Consider the Transformers $T_{\mathrm{Index}}$ and $T_{\mathrm{Retrieval}}$ constructed in Theorem~\ref{teo:mechanisms}. Let $n>1$ denote an arbitrary input length. Then, for $T_\mathrm{Index}$, the discrepancy satisfies $\Delta_\mathrm{Index}(n) \leq \min \left\lbrace1-\cos(\theta),2 \frac{\pi^2}{n^2}\right\rbrace.$ Now let $\omega=2\pi/|\textsc{Alph}|$ and $\theta < \omega/2n_0$. Then, for $T_\mathrm{Retrieval}$  we have that $
    \Delta_\mathrm{Retrieval}(n) \leq 1-\cos\left(\omega - n\theta \right).$ 
%In particular, if $\theta = 0$ (NoPE), then $\Delta_\mathrm{Retrieval}(n)$ remains away from zero for arbitrary lengths. 
\end{theorem}
The proof is provided in Appendix~\ref{proof:teo}. 

\begin{remark}
Theorem \ref{teo:discrepancy} can be thought of as imposing hard limitations on the ability of the model to achieve perfect accuracy when the length of the input increases, for given choices of the corresponding angle in RoPE. Note the strong separation between the Index and the Retrieval functions in terms of how their discrepancies decay with longer inputs. In particular, if $\theta = 0$ (NoPE), then $\Delta_\mathrm{Retrieval}(n)$ remains away from zero for arbitrary lengths. 
\end{remark}

\subsection{Impact of symbolic versus positional mechanisms on sequence length generalization}
\label{sec:generalization}
\begin{figure}[htbp!]
    \centering
    \includegraphics[width=1\linewidth]{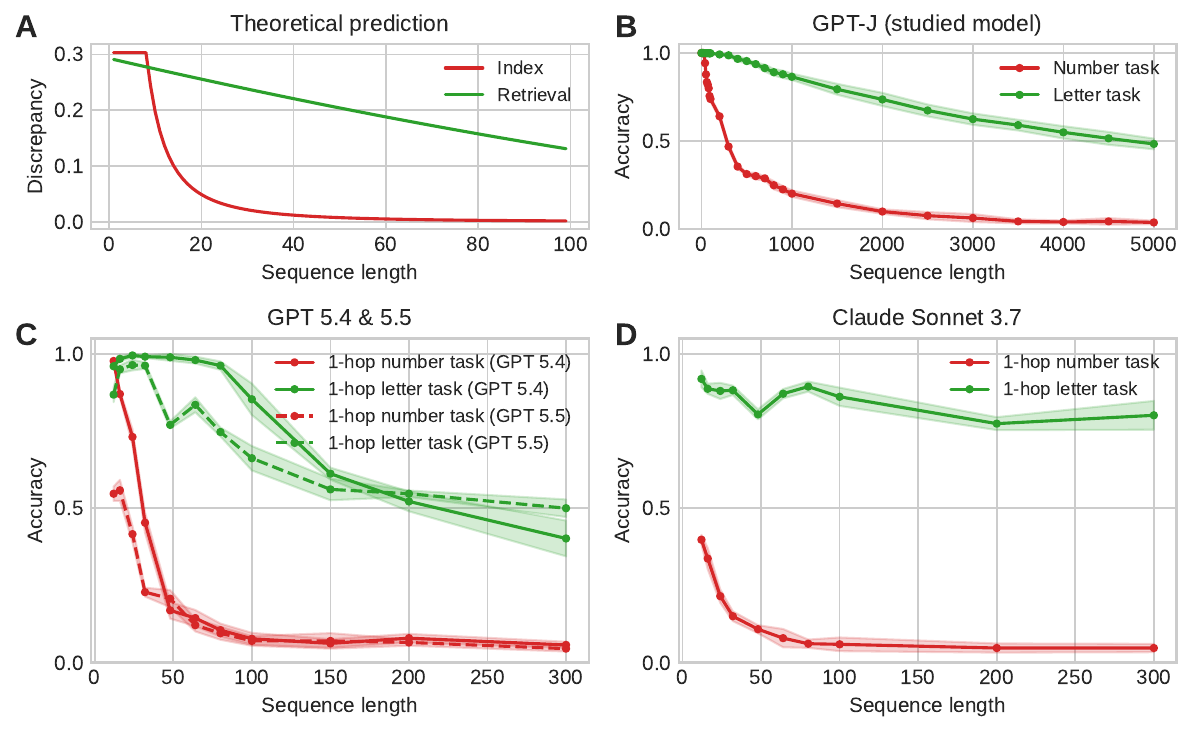}
    \caption{\textbf{Discrepancy and Generalization.} \textbf{A.} Theoretical upper bound of discrepancy for our constructions of the Index (\myred) and Retrieval (\mygreen) functions, for input length $n\leq 100$. We computed the upper-bound discrepancy from the Theorem \ref{teo:discrepancy} using the following parameters: For Index function, $\theta=\theta_2 = 0.8$, while for the Retrieval function, $\theta=\theta_{42} = 0.0027$ and $|\textsc{Alph}| = 8$. \textbf{B.} Accuracy of real models trained on the Number task (\myred) and Letter task (\mygreen), increasing the target window. \textbf{C.} Accuracy of GPT 5.4 and GPT 5.5 (in non-reasoning mode) on a 1-hop version of the  Number (\myred) and Letter (\mygreen) tasks. \textbf{D.} Accuracy of Claude Sonnet 3.7 (in non-reasoning mode) on the same 1-hop versions of the two tasks.}
    \label{fig:discrepancy_and_generalization}
\end{figure}

We now test how our trained models generalize to samples with input lengths that exceed those used during training. Here, we argue that our theoretical analyses, and in particular Theorem \ref{teo:discrepancy}, provide valuable insights to anticipate how each of our trained models should generalize. 

To test their generalization, we extended the target window by adding extra symbols prior to the window containing the to-be-retrieved target symbol. In other words, the target window is itself divided into two parts: an \textit{effective} target window where the response is contained, and an \textit{extra} target window to increase input length. In particular, the difficulty of the tasks remains intuitively unchanged, since the model does not even have to look at the added tokens in order to find the correct answer. This is equally true for both tasks. Yet, Theorem \ref{teo:discrepancy} predicts that the model solving the letter multi-hop task should have far better generalization. To have a better grasp of what it predicts, we have plotted the corresponding theoretical upper bounds for the discrepancy, which should serve as a proxy of generalization behavior. The angles chosen for this plot are taken from the angles used by the trained models. More precisely, we computed the weighted average angle for each layer in each model (see Figure \ref{fig:angles} and Table \ref{table:angles} in Appendix \ref{app:generalization}). The plot is made with the angles corresponding to the second layer in each model.  

As shown in Figure \ref{fig:discrepancy_and_generalization}B, the ability of our trained models for input length generalization on the Letter task (\mygreen ~line) far exceeds that of the Number task (\myred~line). More crucially, the qualitative patterns of decay in the generalization performance closely resemble those predicted by our theoretical model (see Figure \ref{fig:discrepancy_and_generalization}A). Further note that symbolic mechanisms in the letter task allow for near perfect generalization (i.e., over 90\% accuracy) with sequence lengths up to 850 (53 times the original sequence length), whereas the positional mechanism in the number task barely generalizes at all. 

\subsection{Generalization on LLM}

Finally, we performed experiments on frontier LLMs (GPT 5.4, GPT 5.5, Claude Sonnet 3.7) with simplified versions of the letter and number tasks, using sequences of different lengths. Specifically, LLMs were prompted to solve a 1-hop version of each task, with the last token as the query and the remaining tokens as the target window; this 1-hop simplification serves as a proxy for the \emph{Index} and \emph{Retrieval} functions.  Each prompt included instructions in natural language to solve the task, along with few-shot examples (see experimental details in Appendix \ref{app:generalization_llm}). As shown in Figure \ref{fig:discrepancy_and_generalization}C and D, accuracy on the letter task is consistently higher than on the number task. Furthermore, accuracy on the number task decreases rapidly: at sequence length 32, it drops below 0.5 across all tested models, and below 0.1 at sequence length 100. In contrast, the letter task achieves accuracy above 0.88 at length 32 and above 0.65 at length 100 for all models. These results are consistent, once more, with Theorem \ref{teo:discrepancy} and the prediction that performance on the letter task generalizes to longer sequences much better than performance on the number task, suggesting that despite being established in a controlled setting, it still captures phenomena that is relevant to large general-purpose modern LLMs. 

\section{Limitations and Future Work}

{Our analysis uses a controlled single-head-per-layer setting, which reduces within-layer interactions and makes mechanisms easier to attribute. This is a limitation relative to real-world multi-head models, where computations may be distributed across heads, but the transfer of several predictions to real-world models suggests that the simplified setting still captures relevant mechanisms.
Our finite vocabulary also limits extrapolation tests within our controlled setting: in the number task, longer positional hops eventually exhaust the available integer symbols, and in the letter task, longer hop chains exhaust the available letters. Extending the tasks to larger or compositional vocabularies would allow a cleaner separation between sequence length, hop length, and vocabulary size.
Finally, our results point to several future directions. First, models sometimes learn shortcuts that perform more than 1-hop per layer, potentially improving hop generalization, but likely at the cost of reducing robustness to longer inputs. Second, the stronger extrapolation of symbolic mechanisms motivates training schemes that encourage symbolic solutions when both symbolic and positional strategies are available. More broadly, monitoring positional-symbolic scores during the training of larger language models could reveal whether the same mechanisms emerge in practical-scale reasoning tasks, while discrepancy may provide quantitative predictions of when length generalization breaks.}

\newpage

\bibliographystyle{plainnat}
\bibliography{references}

@article{yun2019transformers,
  title={Are transformers universal approximators of sequence-to-sequence functions?},
  author={Yun, Chulhee and Bhojanapalli, Srinadh and Rawat, Ankit Singh and Reddi, Sashank J and Kumar, Sanjiv},
  journal={arXiv preprint arXiv:1912.10077},
  year={2019}
}

@inproceedings{weiss2021thinking,
  title={Thinking like transformers},
  author={Weiss, Gail and Goldberg, Yoav and Yahav, Eran},
  booktitle={International Conference on Machine Learning},
  pages={11080--11090},
  year={2021},
  organization={PMLR}
}

@article{yang2024counting,
  title={Counting like transformers: Compiling temporal counting logic into softmax transformers},
  author={Yang, Andy and Chiang, David},
  journal={arXiv preprint arXiv:2404.04393},
  year={2024}
}

@inproceedings{zhou2023algorithms,
  title={What algorithms can transformers learn? a study in length generalization},
  author={Zhou, Hattie and Bradley, Arwen and Littwin, Etai and Razin, Noam and Saremi, Omid and Susskind, Joshua M and Bengio, Samy and Nakkiran, Preetum},
  booktitle={The Twelfth International Conference on Learning Representations},
  year={2023}
}

@article{kozachinskiy2025strassen,
  title={Strassen Attention, Split VC Dimension and Compositionality in Transformers},
  author={Kozachinskiy, Alexander and Urrutia, Felipe and Jimenez, Hector and Steifer, Tomasz and Pizarro, Germ{\'a}n and Fuentes, Mat{\'\i}as and Meza, Francisco and Calderon, Cristian B and Rojas, Crist{\'o}bal},
  journal={arXiv preprint arXiv:2501.19215},
  year={2025}
}

@article{merrill2022saturated,
  title={Saturated transformers are constant-depth threshold circuits},
  author={Merrill, William and Sabharwal, Ashish and Smith, Noah A},
  journal={Transactions of the Association for Computational Linguistics},
  volume={10},
  pages={843--856},
  year={2022}
}

@article{chu2025llm,
  title={Llm agents for education: Advances and applications},
  author={Chu, Zhendong and Wang, Shen and Xie, Jian and Zhu, Tinghui and Yan, Yibo and Ye, Jinheng and Zhong, Aoxiao and Hu, Xuming and Liang, Jing and Yu, Philip S and others},
  journal={arXiv preprint arXiv:2503.11733},
  year={2025}
}

@article{xu2025towards,
  title={Towards large reasoning models: A survey of reinforced reasoning with large language models},
  author={Xu, Fengli and Hao, Qianyue and Zong, Zefang and Wang, Jingwei and Zhang, Yunke and Wang, Jingyi and Lan, Xiaochong and Gong, Jiahui and Ouyang, Tianjian and Meng, Fanjin and others},
  journal={arXiv preprint arXiv:2501.09686},
  year={2025}
}

@article{welbl2018constructing,
  title={Constructing datasets for multi-hop reading comprehension across documents},
  author={Welbl, Johannes and Stenetorp, Pontus and Riedel, Sebastian},
  journal={Transactions of the Association for Computational Linguistics},
  volume={6},
  pages={287--302},
  year={2018},
  publisher={MIT Press One Rogers Street, Cambridge, MA 02142-1209, USA journals-info~…}
}

@article{ho2020constructing,
  title={Constructing a multi-hop qa dataset for comprehensive evaluation of reasoning steps},
  author={Ho, Xanh and Nguyen, Anh-Khoa Duong and Sugawara, Saku and Aizawa, Akiko},
  journal={arXiv preprint arXiv:2011.01060},
  year={2020}
}

@article{trivedi2020multihop,
  title={Is multihop QA in DiRe condition? Measuring and reducing disconnected reasoning},
  author={Trivedi, Harsh and Balasubramanian, Niranjan and Khot, Tushar and Sabharwal, Ashish},
  journal={arXiv preprint arXiv:2005.00789},
  year={2020}
}

@article{zhong2023mquake,
  title={Mquake: Assessing knowledge editing in language models via multi-hop questions},
  author={Zhong, Zexuan and Wu, Zhengxuan and Manning, Christopher D and Potts, Christopher and Chen, Danqi},
  journal={arXiv preprint arXiv:2305.14795},
  year={2023}
}

@article{zhu2024fanoutqa,
  title={Fanoutqa: Multi-hop, multi-document question answering for large language models},
  author={Zhu, Andrew and Hwang, Alyssa and Dugan, Liam and Callison-Burch, Chris},
  journal={arXiv preprint arXiv:2402.14116},
  year={2024}
}

@misc{sharkey2025openproblemsmechanisticinterpretability,
      title={Open Problems in Mechanistic Interpretability}, 
      author={Lee Sharkey and Bilal Chughtai and Joshua Batson and Jack Lindsey and Jeff Wu and Lucius Bushnaq and Nicholas Goldowsky-Dill and Stefan Heimersheim and Alejandro Ortega and Joseph Bloom and Stella Biderman and Adria Garriga-Alonso and Arthur Conmy and Neel Nanda and Jessica Rumbelow and Martin Wattenberg and Nandi Schoots and Joseph Miller and Eric J. Michaud and Stephen Casper and Max Tegmark and William Saunders and David Bau and Eric Todd and Atticus Geiger and Mor Geva and Jesse Hoogland and Daniel Murfet and Tom McGrath},
      year={2025},
      eprint={2501.16496},
      archivePrefix={arXiv},
      primaryClass={cs.LG},
      url={https://arxiv.org/abs/2501.16496}, 
}

@article{schnitzler2024morehopqa,
  title={Morehopqa: More than multi-hop reasoning},
  author={Schnitzler, Julian and Ho, Xanh and Huang, Jiahao and Boudin, Florian and Sugawara, Saku and Aizawa, Akiko},
  journal={arXiv preprint arXiv:2406.13397},
  year={2024}
}

@article{wu2024mrke,
  title={MRKE: The multi-hop reasoning evaluation of LLMs by knowledge edition},
  author={Wu, Jian and Yang, Linyi and Okumura, Manabu and Zhang, Yue},
  journal={CoRR},
  year={2024}
}

@inproceedings{yang2018hotpotqa,
  title={HotpotQA: A dataset for diverse, explainable multi-hop question answering},
  author={Yang, Zhilin and Qi, Peng and Zhang, Saizheng and Bengio, Yoshua and Cohen, William and Salakhutdinov, Ruslan and Manning, Christopher D},
  booktitle={Proceedings of the 2018 conference on empirical methods in natural language processing},
  pages={2369--2380},
  year={2018}
}

@article{cambon2023early,
  title={Early LLM-based tools for enterprise information workers likely provide meaningful boosts to productivity},
  author={Cambon, Alexia and Hecht, Brent and Edelman, Ben and Ngwe, Donald and Jaffe, Sonia and Heger, Amy and Vorvoreanu, Mihaela and Peng, Sida and Hofman, Jake and Farach, Alex and others},
  journal={Microsoft Research. MSR-TR-2023},
  volume={43},
  year={2023}
}

@article{ameisen2025circuit,
  title={Circuit tracing: Revealing computational graphs in language models},
  author={Ameisen, Emmanuel and Lindsey, Jack and Pearce, Adam and Gurnee, Wes and Turner, Nicholas L and Chen, Brian and Citro, Craig and Abrahams, David and Carter, Shan and Hosmer, Basil and others},
  journal={Transformer Circuits Thread},
  volume={6},
  year={2025}
}

@article{barbero2024round,
	title={Round and round we go! what makes rotary positional encodings useful?},
	author={Barbero, Federico and Vitvitskyi, Alex and Perivolaropoulos, Christos and Pascanu, Razvan and Veli{\v{c}}kovi{\'c}, Petar},
	journal={arXiv preprint arXiv:2410.06205},
	year={2024}
}

@article{xiong2023effective,
	title={Effective long-context scaling of foundation models},
	author={Xiong, Wenhan and Liu, Jingyu and Molybog, Igor and Zhang, Hejia and Bhargava, Prajjwal and Hou, Rui and Martin, Louis and Rungta, Rashi and Sankararaman, Karthik Abinav and Oguz, Barlas and others},
	journal={arXiv preprint arXiv:2309.16039},
	year={2023}
}

@article{ding2024longrope,
	title={Longrope: Extending llm context window beyond 2 million tokens},
	author={Ding, Yiran and Zhang, Li Lyna and Zhang, Chengruidong and Xu, Yuanyuan and Shang, Ning and Xu, Jiahang and Yang, Fan and Yang, Mao},
	journal={arXiv preprint arXiv:2402.13753},
	year={2024}
}

@article{chen2023extending,
	title={Extending context window of large language models via positional interpolation},
	author={Chen, Shouyuan and Wong, Sherman and Chen, Liangjian and Tian, Yuandong},
	journal={arXiv preprint arXiv:2306.15595},
	year={2023}
}

@article{hahn2020theoretical,
	title={Theoretical limitations of self-attention in neural sequence models},
	author={Hahn, Michael},
	journal={Transactions of the Association for Computational Linguistics},
	volume={8},
	pages={156--171},
	year={2020},
	publisher={MIT Press One Rogers Street, Cambridge, MA 02142-1209, USA journals-info~…}
}

@article{peng2024limitations,
	title={On limitations of the transformer architecture},
	author={Peng, Binghui and Narayanan, Srini and Papadimitriou, Christos},
	journal={arXiv preprint arXiv:2402.08164},
	year={2024}
}

@article{bhattamishra2024separations,
	title={{Separations in the Representational Capabilities of Transformers and Recurrent Architectures}},
	author={Bhattamishra, Satwik and Hahn, Michael and Blunsom, Phil and Kanade, Varun},
	journal={arXiv preprint arXiv:2406.09347},
	year={2024}
}

@article{strobl2024formal,
	title={What formal languages can transformers express? a survey},
	author={Strobl, Lena and Merrill, William and Weiss, Gail and Chiang, David and Angluin, Dana},
	journal={Transactions of the Association for Computational Linguistics},
	volume={12},
	pages={543--561},
	year={2024},
	publisher={MIT Press 255 Main Street, 9th Floor, Cambridge, Massachusetts 02142, USA~…}
}

@article{su2024roformer,
	title={Roformer: Enhanced transformer with rotary position embedding},
	author={Su, Jianlin and Ahmed, Murtadha and Lu, Yu and Pan, Shengfeng and Bo, Wen and Liu, Yunfeng},
	journal={Neurocomputing},
	volume={568},
	pages={127063},
	year={2024},
	publisher={Elsevier}
}

@article{pasten2025continuity,
	title={Continuity and Isolation Lead to Doubts or Dilemmas in Large Language Models},
	author={Pasten, Hector and Urrutia, Felipe and Jimenez, Hector and Calderon, Cristian B and Rojas, Crist{\'o}bal and Kozachinskiy, Alexander},
	journal={arXiv preprint arXiv:2505.10606},
	year={2025}
}

@article{urrutia2025decoupling,
  title={Decoupling Positional and Symbolic Attention Behavior in Transformers},
  author={Urrutia, Felipe and Salas, Jorge and Kozachinskiy, Alexander and Calderon, Cristian Buc and Pasten, Hector and Rojas, Cristobal},
  journal={arXiv preprint arXiv:2511.11579},
  year={2025}
}

@article{wu2025transformers,
  title={How Do Transformers Learn Variable Binding in Symbolic Programs?},
  author={Wu, Yiwei and Geiger, Atticus and Milli{\`e}re, Rapha{\"e}l},
  journal={arXiv preprint arXiv:2505.20896},
  year={2025}
}

@article{conmy2023towards,
  title={Towards automated circuit discovery for mechanistic interpretability},
  author={Conmy, Arthur and Mavor-Parker, Augustine and Lynch, Aengus and Heimersheim, Stefan and Garriga-Alonso, Adri{\`a}},
  journal={Advances in Neural Information Processing Systems},
  volume={36},
  pages={16318--16352},
  year={2023}
}

@article{elhage2021mathematical,
   title={A Mathematical Framework for Transformer Circuits},
   author={Elhage, Nelson and Nanda, Neel and Olsson, Catherine and Henighan, Tom and Joseph, Nicholas and Mann, Ben and Askell, Amanda and Bai, Yuntao and Chen, Anna and Conerly, Tom and DasSarma, Nova and Drain, Dawn and Ganguli, Deep and Hatfield-Dodds, Zac and Hernandez, Danny and Jones, Andy and Kernion, Jackson and Lovitt, Liane and Ndousse, Kamal and Amodei, Dario and Brown, Tom and Clark, Jack and Kaplan, Jared and McCandlish, Sam and Olah, Chris},
   year={2021},
   journal={Transformer Circuits Thread},
   note={https://transformer-circuits.pub/2021/framework/index.html}
}

@inproceedings{hegde2024effectiveness,
  title={Effectiveness of Sparse Autoencoder for understanding and removing gender bias in LLMs},
  author={Hegde, Praveen},
  booktitle={NeurIPS 2024 Workshop on Scientific Methods for Understanding Deep Learning}
}

@article{bereska2024mechanistic,
  title={Mechanistic interpretability for AI safety--a review},
  author={Bereska, Leonard and Gavves, Efstratios},
  journal={arXiv preprint arXiv:2404.14082},
  year={2024}
}

@article{cammarata2025painting,
  author = {Cammarata, Nick and Bissell, Mark and Nguyen, Nam and Deng, Myra and Ho, Eric and Gorton, Liv and Loeffler, Max and Balsam, Daniel},
  title = {Painting with concepts using diffusion model latents},
  journal = {Goodfire},
  year = {2025},
  note = {https://paint.goodfire.ai/}
}

@article{lindsey2025landscape,
   author={Lindsey, Jack and Ameisen, Emmanuel and Nanda, Neel and Shabalin, Stepan and Piotrowski, Mateusz and McGrath, Tom and Hanna, Michael and Lewis, Owen and Tigges, Curt and Merullo, Jack and Watts, Connor and Paulo, Gonçalo and Batson, Joshua and Gorton, Liv and Simon, Elana and Loeffler, Max and McDougall, Callum and Lin, Johnny},
   title={The Circuits Research Landscape: Results and Perspectives},
   journal={Neuronpedia},
   year={2025},
   url={https://neuronpedia.org/graph/info}
}

@article{liao2021hop,
  title={To hop or not, that is the question: Towards effective multi-hop reasoning over knowledge graphs},
  author={Liao, Jinzhi and Zhao, Xiang and Tang, Jiuyang and Zeng, Weixin and Tan, Zhen},
  journal={World Wide Web},
  volume={24},
  number={5},
  pages={1837--1856},
  year={2021},
  publisher={Springer}
}

@inproceedings{wu2025mmqa,
  title={MMQA: Evaluating LLMs with multi-table multi-hop complex questions},
  author={Wu, Jian and Yang, Linyi and Li, Dongyuan and Ji, Yuliang and Okumura, Manabu and Zhang, Yue},
  booktitle={The Thirteenth International Conference on Learning Representations},
  pages={1},
  year={2025}
}

@article{zhao2024benchmarking,
  title={Benchmarking multi-image understanding in vision and language models: Perception, knowledge, reasoning, and multi-hop reasoning},
  author={Zhao, Bingchen and Zong, Yongshuo and Zhang, Letian and Hospedales, Timothy},
  journal={arXiv preprint arXiv:2406.12742},
  year={2024}
}

@inproceedings{wan2021gaussianpath,
  title={Gaussianpath: A bayesian multi-hop reasoning framework for knowledge graph reasoning},
  author={Wan, Guojia and Du, Bo},
  booktitle={Proceedings of the AAAI conference on artificial intelligence},
  volume={35},
  number={5},
  pages={4393--4401},
  year={2021}
}

@inproceedings{shi2022stepgame,
  title={Stepgame: A new benchmark for robust multi-hop spatial reasoning in texts},
  author={Shi, Zhengxiang and Zhang, Qiang and Lipani, Aldo},
  booktitle={Proceedings of the AAAI conference on artificial intelligence},
  volume={36},
  number={10},
  pages={11321--11329},
  year={2022}
}

@article{guo2025llms,
  title={How Do LLMs Perform Two-Hop Reasoning in Context?},
  author={Guo, Tianyu and Zhu, Hanlin and Zhang, Ruiqi and Jiao, Jiantao and Mei, Song and Jordan, Michael I and Russell, Stuart},
  journal={arXiv preprint arXiv:2502.13913},
  year={2025}
}

@inproceedings{hou-etal-2023-towards,
    title = "Towards a Mechanistic Interpretation of Multi-Step Reasoning Capabilities of Language Models",
    author = "Hou, Yifan  and
      Li, Jiaoda  and
      Fei, Yu  and
      Stolfo, Alessandro  and
      Zhou, Wangchunshu  and
      Zeng, Guangtao  and
      Bosselut, Antoine  and
      Sachan, Mrinmaya",
    editor = "Bouamor, Houda  and
      Pino, Juan  and
      Bali, Kalika",
    booktitle = "Proceedings of the 2023 Conference on Empirical Methods in Natural Language Processing",
    month = dec,
    year = "2023",
    address = "Singapore",
    publisher = "Association for Computational Linguistics",
    url = "https://aclanthology.org/2023.emnlp-main.299/",
    doi = "10.18653/v1/2023.emnlp-main.299",
    pages = "4902--4919"
}

@article{meng2022locating,
  title={Locating and editing factual associations in gpt},
  author={Meng, Kevin and Bau, David and Andonian, Alex and Belinkov, Yonatan},
  journal={Advances in neural information processing systems},
  volume={35},
  pages={17359--17372},
  year={2022}
}

@inproceedings{langedijk-etal-2024-decoderlens,
    title = "{D}ecoder{L}ens: Layerwise Interpretation of Encoder-Decoder Transformers",
    author = "Langedijk, Anna  and
      Mohebbi, Hosein  and
      Sarti, Gabriele  and
      Zuidema, Willem  and
      Jumelet, Jaap",
    editor = "Duh, Kevin  and
      Gomez, Helena  and
      Bethard, Steven",
    booktitle = "Findings of the Association for Computational Linguistics: NAACL 2024",
    month = jun,
    year = "2024",
    address = "Mexico City, Mexico",
    publisher = "Association for Computational Linguistics",
    url = "https://aclanthology.org/2024.findings-naacl.296/",
    doi = "10.18653/v1/2024.findings-naacl.296",
    pages = "4764--4780"
}

@inproceedings{samaran-etal-2021-attending,
    title = "Attending Self-Attention: A Case Study of Visually Grounded Supervision in Vision-and-Language Transformers",
    author = "Samaran, Jules  and
      Garcia, Noa  and
      Otani, Mayu  and
      Chu, Chenhui  and
      Nakashima, Yuta",
    editor = "Kabbara, Jad  and
      Lin, Haitao  and
      Paullada, Amandalynne  and
      Vamvas, Jannis",
    booktitle = "Proceedings of the 59th Annual Meeting of the Association for Computational Linguistics and the 11th International Joint Conference on Natural Language Processing: Student Research Workshop",
    month = aug,
    year = "2021",
    address = "Online",
    publisher = "Association for Computational Linguistics",
    url = "https://aclanthology.org/2021.acl-srw.8/",
    doi = "10.18653/v1/2021.acl-srw.8",
    pages = "81--86"
}

@article{pan2026opening,
  title={Opening the Black Box: A Survey on the Mechanisms of Multi-Step Reasoning in Large Language Models},
  author={Pan, Liangming and Liang, Jason and Ye, Jiaran and Yang, Minglai and Lu, Xinyuan and Zhu, Fengbin},
  year={2026},
  publisher={Preprints}
}

@article{abnar2020quantifying,
  title={Quantifying attention flow in transformers},
  author={Abnar, Samira and Zuidema, Willem},
  journal={arXiv preprint arXiv:2005.00928},
  year={2020}
}

@article{gur2025mixing,
  title={Mixing Mechanisms: How Language Models Retrieve Bound Entities In-Context},
  author={Gur-Arieh, Yoav and Geva, Mor and Geiger, Atticus},
  journal={arXiv preprint arXiv:2510.06182},
  year={2025}
}

@misc{mesh-transformer-jax,
  author = {Wang, Ben},
  title = {{Mesh-Transformer-JAX: Model-Parallel Implementation of Transformer Language Model with JAX}},
  howpublished = {\url{https://github.com/kingoflolz/mesh-transformer-jax}},
  year = 2021,
  month = May
}

%%%%%%%%%%%%%%%%%%%%%%%%%%%%%%%%%%%%%%%%%%%%%%%%%%%%%%%%%%%%

\appendix
\section{Related Work}
\label{sec:relatedwork}

Given the incremental performance of Large Reasoning Models \citep{xu2025towards}, several studies have focused on gaining a deeper understanding of the mechanisms underlying the solution to simpler multi-hop reasoning tasks. The goal of these studies is to strike a good balance between task complexity, model complexity and interpretability in order to provide insights that can help us understand how Transformer-based models are solving more complex tasks at a practical scale.

\paragraph{Multi-hop Reasoning.} Multi-hop reasoning tasks require chaining multiple intermediate computations before being able to infer a final answer \citep{yang2018hotpotqa, welbl2018constructing}. To evaluate the reasoning abilities of LLMs, several multi-hop question answering benchmarks have been developed over the past years \citep{ho2020constructing, trivedi2020multihop, zhong2023mquake,wu2024mrke, zhu2024fanoutqa,schnitzler2024morehopqa,shi2022stepgame}. The modality or type of data upon which the multi-hop task is being implemented can vary. For example, \citet{zhao2024benchmarking} investigate multi-hop reasoning over images, and find that such tasks are more difficult compared with their text counterpart. In turn, \citet{wan2021gaussianpath} focus on knowledge graph answering. They proposed a Bayes-based RL multi-hop reasoning framework that can explicitly quantify the uncertainty over reasoning sequences. More recently, \citet{wu2025mmqa} examined the models' reasoning abilities over multi-table question answering. Notably, they showed that modern LLMs still struggle in solving these kind of tasks, specially for larger hops structure and input lengths. Whereas these benchmarks can shed light on abilities, they are often too complex to be interpretable. Such limitations sparked research on these multi-hop tasks \citep{hou-etal-2023-towards, pan2026opening, guo2025llms}, but in interpretable settings. Here, we go a step further, and focus not only on how Transformer-based trained models solve these tasks, but also on how learning the properties necessary to solve these tasks give us crucial insights to potentially improve the models \citep{wu2025transformers}.

\paragraph{Methods for circuits and mechanisms interpretability.}
Typically, interpretability techniques are applied over attention layers and internal representations~\citep{samaran-etal-2021-attending, abnar2020quantifying, elhage2021mathematical,langedijk-etal-2024-decoderlens}. The main goal of these methods is to understand the role of a given hidden layer in the distinct mechanisms that models implement. For example, Linear Probing focuses on predicting a model's output from an intermediate layer's hidden state by using simple linear models \citep{wu2025transformers}. Similarly, cross layer transcoders are interconnected encoders used to predict a layers output from all previous activations, clarifying the impact of a particular layer on the following ones \citep{ameisen2025circuit}. In turn, activation patching implements local interventions to hidden vectors while tracking changes in the output \citep{meng2022locating,gur2025mixing}. Although useful, these methods have three important limitations: \textit{(i)} high computation and training times, \textit{(ii)} vague interpretation of features and \textit{(iii)} locality of interventions. In this work we opt for a novel method that is efficiently computable (i.e., does not require training another model), allows to interpret attention mechanisms in a task-agnostic way (not requiring causal models) and summarizes the impact of several query positions (i.e., global rather than local interventions).

\paragraph{Positional and Symbolic mechanisms.}

\citet{barbero2024round} observed two distinct mechanisms in RoPE-based Transformers. When models require implementing symbolic mechanism to solve a task, attention heads tend to use lower RoPE frequencies. In contrast, when the task requires the implementation of positional mechanisms, attention heads prefer the use of mid to high frequencies. \citet{gur2025mixing} observed that symbolic (or reflexive) and positional mechanisms underlay the solution to a relational question answering task. \citet{urrutia2025decoupling} further revealed a tight link between positional and symbolic attention mechanisms and how queries and keys vector interact in the model. Furthermore, they showed that symbolic mechanisms (using lower frequencies) are crucial for better performance in longer contexts. Such results are in line with previous works showing that increasing the RoPE base frequency improves longer contexts performance \citep{xiong2023effective,chen2023extending,ding2024longrope}. 

\paragraph{Expressivity results.} A natural approach to understanding the behavior of Transformers on structured reasoning tasks is through the lens of expressivity and algorithmic mechanisms. Early results established that Transformers are universal approximators of sequence-to-sequence functions under sufficient capacity assumptions \citep{yun2019transformers}. Subsequent work showed that realistic architectural constraints such as bounded depth, restricted precision, or specific attention mechanisms induce important computational limitations \citep{hahn2020theoretical,merrill2022saturated,strobl2024formal}.

Most prior theoretical analyses have focused on formal language recognition or asymptotic computational power. For instance, hardmax Transformers with constant depth cannot recognize languages such as PARITY or Dyck-1 \citep{hahn2020theoretical}, while saturated Transformers can be characterized through constant-depth threshold circuits \citep{merrill2022saturated}. More recently, communication-complexity and VC-dimension-based techniques techniques have been used to derive lower bounds for softmax Transformers on compositional and matching tasks \citep{peng2024limitations,bhattamishra2024separations,kozachinskiy2025strassen}.

Closer to our setting, several works have studied the ability Transformers to implement programming languages. \cite{weiss2021thinking} analyzed how attention layers emulate primitives of RASP \cite{weiss2021thinking}, while \cite{yang2024counting} showed that Transformers can implement temporal counting mechanisms. \cite{zhou2023algorithms} suggested that Transformers may learn solutions that generalize across input lengths when the task admits a short RASP-L program valid for arbitrary sequence lengths.

Our work is closely related in spirit to the previously described studies, but differs in a key aspect. Most (if not all) of the prior positive results in the expressivity literature are established for models equipped with idealized or ad hoc components, such as hard attention or specially designed positional encodings. While highly valuable, these theoretical constructions avoid tackling the issue of what mechanisms Transformers actually learn to implement. Our work aims to close this gap. First, we characterize the concrete mechanisms that emerge during learning in RoPE-based models, through the distinction between positional and symbolic attention heads. Second, we connect geometric properties of our RoPE-based constructions to measurable failures in real models length generalization.

\section{Multi-hop task formalization}
\label{app:tasks}

Let $m$ be the number of integers. To formally define the tasks, we consider a finite set of alphabetic symbols $\textsc{Alph}$, a set of integers $\textsc{Int} = [m]$, a target window length $n_1>0$, a context length $n_2>0$ and hops length $\textsc{HL}>0$.
\subsection{Number Task}
We consider inputs $\posinp:= s_1s_2...s_{n_1}j_1j_2...j_{n_2}$ such that $s_i$ is a symbol from $\textsc{Alph}$ for every $i=1,...,n_1$ and $j_i$ is an integer from $\textsc{Int}$ for every $i=1,...,n_2$. For a given input $\posinp$, we assume there exists a sequence of integers $j_{k_1}j_{k_2}...j_{k_{\textsc{HL}}}\subseteq j_{n_2}j_{{n_2}-1}...j_{1}$ such that:
\begin{enumerate}
    \item $j_{k_1} = j_{n_2}$.
    
    \item $n_1 - j_{{k_{\textsc{HL}}}} + 1\in\{1,2,...,n_1\}$.

    \item For every $i=1,..,\textsc{HL}-1$ then $k_i - k_{i+1} = j_{k_i}$.
\end{enumerate}

The solution of the number task is a function $\fpos$ such that for every $\posinp$ then $\fpos(\posinp) = s_{{n_2}-j_{k_{\textsc{HL}}}+1}$ (see Figure \ref{fig:main_hop} left). Intuitively, an input $\posinp$ includes a sequence of integers where each number defines the position of the following hop in the sequence. We assume the first element of this sequence is the last integer in $\posinp$, while the last element points to the relative position of a symbol $s^{*}\in \{s_1,...,s_{n_1}\}$. We call this task {\em{number-based multi-hop}} since the relevant information to solve the task is the token number indicating the hop sequence, and not the letters in the sequence.

\noindent\textbf{Implementation.} Our dataset consisted of 480,000 sequences, each containing 17 tokens: a target window and a pretarget window of 8 tokens each, followed by a single query token. The \textsc{Alph} vocabulary contained 120 tokens, comprising the 26 single letters $a, b, \dots, z$ together with 94 character bigram tokens ranging from $aa$ to $dp$ in lexicographic order. The \textsc{Int} vocabulary contained the 16 integer tokens \{1, 2, \dots, 16\}. The dataset therefore contained 120 + 16 = 136 unique tokens in total. Sequences were drawn in equal proportion from 1-, 2-, 3-, and 4-hop instances, and within each hop class we balanced both the value of the solution token and its position in the target window. The dataset was partitioned into training, validation, and test sets in a 90 / 0.5 / 9.5 percentage split (432,000 / 2,400 / 45,600 sequences).

\subsection{Letter task}
We consider inputs $\syminp:= (s_1, w_1)(s_2,w_2)...(s_{n_1},w_{n_1})(x_1,y_1)(x_2,y_2)...(x_{n_2}, y_{n_2})$ where $s_i, x_i$ and $y_i$ are symbols from $\textsc{Alph}$ for every $i=1,...,\max\{n_1,n_2\}$ while $w_i\in\textsc{Int}$ for $i=1,...,n_1$. For an input $\syminp$ we assume there exist a sequence $(x_{k_1}, y_{k_1})(x_{k_2}, y_{k_2})...(x_{k_{\textsc{HL}}}, y_{k_{\textsc{HL}}}) \subseteq (x_{n_2}, y_{n_2})...(x_1, y_1) $ such that:

\begin{enumerate}
    \item $(x_{k_1},y_{k_1}) = (x_{n_2},y_{n_2})$.
    \item There exists $i^{*}\in\{1,...,n_1\}$ such that $y_{k_H}=s_{i^{*}}$.
    \item For every $i=1,...,\textsc{HL}-1$ then $y_{k_i} = x_{k_{i+1}}$.
\end{enumerate}

The solution for this task is a function $\fsym$ such that for every input $\syminp$ then $\fsym(\syminp) = (s_{i^*}, w_{i^{*}})$ (see Figure \ref{fig:main_hop} middle). Intuitively, a solution model for $\fsym$ should follow the sequence of {\em letter hops} given by $(x_{k_1}, y_{k_1})(x_{k_2}, y_{k_2})...(x_{k_{\textsc{HL}}}, y_{k_{\textsc{HL}}})$ and retrieve the token $(s_{i^*}, w_{i^{*}})$ pointed by the second element of the tuple $(x_{k_{\textsc{HL}}}, y_{k_{\textsc{HL}}})$. We name this task \textit{letter-based multi-hop} given that letters define the hop sequence to solve the task.

\noindent\textbf{Implementation.} The letter task dataset mirrored the number task dataset in structure: 480,000 sequences of 17 tokens each (an 8-token target window, an 8-token pretarget window, and a single query token). The \textsc{Alph} vocabulary contained the 8 letters a,b,…,ha, b, \dots, h, and the \textsc{Int} vocabulary the 8 integers \{1, 2, \dots, 8\}. Each target window token was a pair in \textsc{Alph} $\times$ \textsc{Int} and each pretarget-window token a pair in \textsc{Alph} $\times$ \textsc{Alph}, with each pair treated as a single atomic token in the model's vocabulary. This yields 64 + 64 = 128 unique tokens in total, comparable in size to the 136-token vocabulary of the number task. The same balancing procedure was applied: equal proportions of 1-, 2-, 3-, and 4-hop instances, with the value and position of the solution balanced within each hop class. The train/validation/test split was also identical to that of the number task.

\section{Transformer preliminaries, GPT-J Architecture, and training details}
\label{app:architecture}

\subsection{Transformer preliminaries}

\begin{definition}[Attention Head]
\label{def_layer}
    A $d$-dimensional decoder-only \emph{attention head} is a function $\Hcal\colon(\mathbb{R}^d)^* \to(\mathbb{R}^d)^*$ given by a continuous ``Logits function'' $L\colon (\mathbb{R}^d\times \mathbb{N})^2 \to \mathbb{R}$, a continuous ``Value function'' $\val\colon\mathbb{R}^d\to\mathbb{R}^d$, and a continuous ``activation function'' $F\colon\mathbb{R}^d\times\mathbb{R}^d \to\mathbb{R}^d$.
    Given an input sequence of vectors $\bar x = (x_1, \ldots, x_n)\in(\mathbb{R}^d)^n$, the head $H$ outputs a sequence of vectors $\bar y = (y_1, \ldots, y_n) = H(\bar x)$, computed as follows.     
    First, the ``attention weights'' are computed as $ 
w_j^i = \exp(\lambda_j^i)/\sum_{k\leq i}\exp(\lambda_k^i), \,\, \text{with} \,\, \lambda_j^i=L(x_i, i,x_j,j);$ then, a sequence $\bar a = (a_1, \ldots, a_n)$ of ``attention vectors'': ${a_i = \sum_{j=1}^i w_j^i\cdot \val(x_j)}$;
 and finally, one sets: $y_i = F(a_i, x_i)$, $i = 1,\ldots, n.$
\end{definition}

\noindent{\textbf{Rotary Positional Encoding (RoPE)}}. We assume an even hidden dimension $d\geq2$. For RoPE \cite{su2024roformer} we consider the set of angles $G = (g_k = \theta^{-2(k-1)/d} : k=1,...,d/2)$ where $g_1$ is the fastest rotation angle/frequency and $g_{d/2}$ is the slowest one. For a given angle $g_k$, we denote the rotation matrix $\rho(g_k)$ as 
$$\rho(g_k)=\left[\begin{matrix}
\cos(g_k) & -\sin(g_k) \\
\sin(g_k) & \cos(g_k) 
\end{matrix}\right].$$

To account for different rotation angles, we define a matrix $\mathbf{R}^{i}=\bigoplus_{k=1}^{d/2} \rho(g_{k})^{i} \in \mathbb{R}^{d\times d}$ as the concatenation of the $i$-th exponentiation of rotation matrices $\rho(g_k)$ in a block diagonal matrix. Since each matrix $\rho(g_k)$ is a rotation matrix , we get that $\rho(g_k)^{i} = \rho(i g_k)$. Finally, for a query vector $x_i$ at position $i$ and key vector $x_j$ at position $j$, we define the RoPE logits function as: $$L_{\text{RoPE}}(x_i,i,x_j,j) = (\mathbf{R}^{j}x_j)^{\top}(\mathbf{R}^{i}x_i) = x_{j}^{\top}\mathbf{R}^{i-j}x_{i}.$$

\subsection{GPT-J architecture and training details}

We train from scratch a decoder-only Transformer model based on GPT-J \cite{mesh-transformer-jax}. Our model is configured with 12 layers, each with a single attention head, and a hidden dimension of 128. We encode positional information using RoPE, apply GELU activations between layers, and a dropout rate of $0.1$. Unlike standard Causal Language Modeling which optimizes the joint probability over the entire sequence, we employ a \emph{last-token prediction} strategy. Let $\bar{x} = (x_1, \ldots, x_n)$ be a sequence of fixed length $n$, where $x_1, \ldots, x_{n-1}$ is the input context and $x_n$ represents the target label. We optimize the model parameters by minimizing the negative log-likelihood exclusively at the terminal step $n$. This is achieved by masking the gradient contribution from all preceding tokens. 
This is implemented by constructing a target vector $y$ where we apply the ignore index ($-100$) to all positions $t < n$, ensuring that only the final token $y_n = x_n$ contributes to the gradient. The objective function for a single sample is defined as:
\[
\mathcal{L}(\theta) = - \log P_\theta(x_n \mid x_1, \ldots, x_{n-1})
\]

\section{Positional and Symbolic Scores}
\label{app:scores}
Following \cite{urrutia2025decoupling}, we introduce the relevant definitions.

We start by providing formal definitions of what it means for a head to behave positionally or symbolically on a given input. Given an arbitrary head $\Hcal$ with logits function $L$ and an input $\bar{x}$, we define the logits produced when querying the final token $x_n$ as
$
L(\overline{x}) 
= \bigl( L(x_n, n, x_j, j) : j \leq n \bigr).$
Then, we define the corresponding attention weights as $D(\overline{x}) = \softmax(L(\overline{x}))
$. In the following, $S_i$ will denote the set of permutations of $[i]$. 

\begin{definition}[Positional and Symbolic attention]\label{def:act_pos_sym}
We say that a head $\Hcal$ acts \emph{\textbf{positionally}} on an input $\overline{x} = (x_1,x_2,...,x_n)$ at query $i\leq n$ if its logit sequence is \emph{{invariant}} under permutation of the key vectors. That is, if for all $\pi\in S_{i-1}$:
\[
L(x_i,i,x_{\pi(j)},j)=L(x_i,i,x_j,j) \qquad \forall j<i.
\]
%\[
% \bar{\lambda}^i(\pi(\bar{x})) = \bar{\lambda}^i(\bar{x}) \qquad\forall \pi\in S_{i-1}.  
%\]
On the other hand, we will say that $\Hcal$ acts \emph{\textbf{symbolically}} on $\overline{x} = (x_1,x_2,...,x_{n})\in(\mathbb{R}^{d})^n$ at query $i\leq n$ if its logit sequence is \emph{{equivariant}} under permutations of the key vectors. That is, if 
\[
L(x_i,i,x_j,\pi(j))=L(x_i,i,x_j,j) \qquad \forall \pi \in S_{i-1}, \quad \forall j<i.
\]
%\[
%\bar{\lambda}^i(\pi(\bar{x})) = \pi(\bar{\lambda}^i(\bar{x})) \qquad \forall \pi \in S_{i-1}.
%\]
\end{definition}

In other words, suppose we are querying on a given input from position $i$ to all positions $j\leq i$. A head $\Hcal$ is positional when $L(x_i,i,x_j,j)$ depends only on the position $j$, and not on the value of the key vector $x_j$. In contrast, $\Hcal$ will be symbolic when $L(x_i,i,x_j,j)$ only depends on the key-vector $x_j$ (the ``symbol''), regardless of the position $j$.

\paragraph{Two--dimensional representation.}
A simple type of permutation $\pi$ is a swap between block $i$ and block $j$. Restricting to such permutations is no loss of generality since every permutation can be decomposed in a composition of swaps. We represent the relevant attention mass by the vector
\[
v_{ij}(\overline{x}) = (D_i(\overline{x}),\, D_j(\overline{x})),
\]
and after the permutation by
\[
v_{ij}(\pi(\overline{x})) = (D_i(\pi(\overline{x})),\, D_j(\pi(\overline{x}))).
\]

\paragraph{Permutation weights.}
Not all swaps are equally informative: we weight each permutation $\pi$ according to the amount of mass it moves,
\[
\alpha(\pi) = \softmax \!\bigl(\bigl| d_i(\overline{x}) - d_j(\overline{x}) \bigr|/\tau\bigr),
\]
where $\tau > 0$ is a temperature parameter controlling how sharply the softmax focuses on high-mass swaps.

\paragraph{Positional and symbolic scores.}
We can now quantify how the head behaves with respect to swap permutations using cosine similarity\footnote{$\cos \text{sim}(a, b) \;=\; \frac{\langle a, b \rangle}{\|a\|\,\|b\|}$.}.  The \emph{positional score} measures how stable the logits remain under a permutation:
\[
s_\textsc{pos}(\overline{x}) 
= \sum_{\pi} \alpha(\pi)\,
\cos \text{sim}\Bigl(v_{ij}(\pi(\overline{x})),\, v_{ij}(\overline{x})\Bigr).
\]
The \emph{symbolic score} instead measures whether each logit moves exactly as the permutation prescribes:
\[
s_\textsc{sym}(\overline{x}) 
= \sum_{\pi} \alpha(\pi)\,
\cos \text{sim}\Bigl(v_{ij}(\pi(\overline{x})),\,v_{ji}(\overline{x})\Bigr).
\]

 Now, for a set of tuples of hidden vectors $S=\{\overline{x}_1,...,\overline{x}_{|S|}\}$ and attention head $H$, we define the positional score of $H$ over $S$ as the average value $$\spos(S,H) = \displaystyle\frac{1}{|S|}\displaystyle\sum_{\overline{x}_k \in S} \spos(\overline{x}_k, H).$$

Similarly, we define the symbolic score over a set of hidden vectors $S$ as the average value

$$\ssym(S,H) = \displaystyle\frac{1}{|S|}\displaystyle\sum_{\overline{x}_k\in S} \ssym(\overline{x}_k, H).$$

\section{Learning dynamics in all layers}
\label{app:full_results}

Figure \ref{fig:pos_sym_all_layers} illustrates the symbolic and positional scores across each layer for the models trained on the number and letter tasks, respectively. Additionally, we display the per-layer entropy, computed from the attention logits of the last token in the input sequence. These entropy values are subsequently subjected to min-max normalization at each layer. Figure~\ref{fig:gammas} shows the cumulative distribution of head purity using $\gamma \in \{0.1, 0.05\}$. Note that while smaller values are more restrictive, the shape of the curve remains consistent, aside from a small shift along the $x$-axis. 

\begin{figure}
    \centering
    \includegraphics[width=0.7\linewidth]{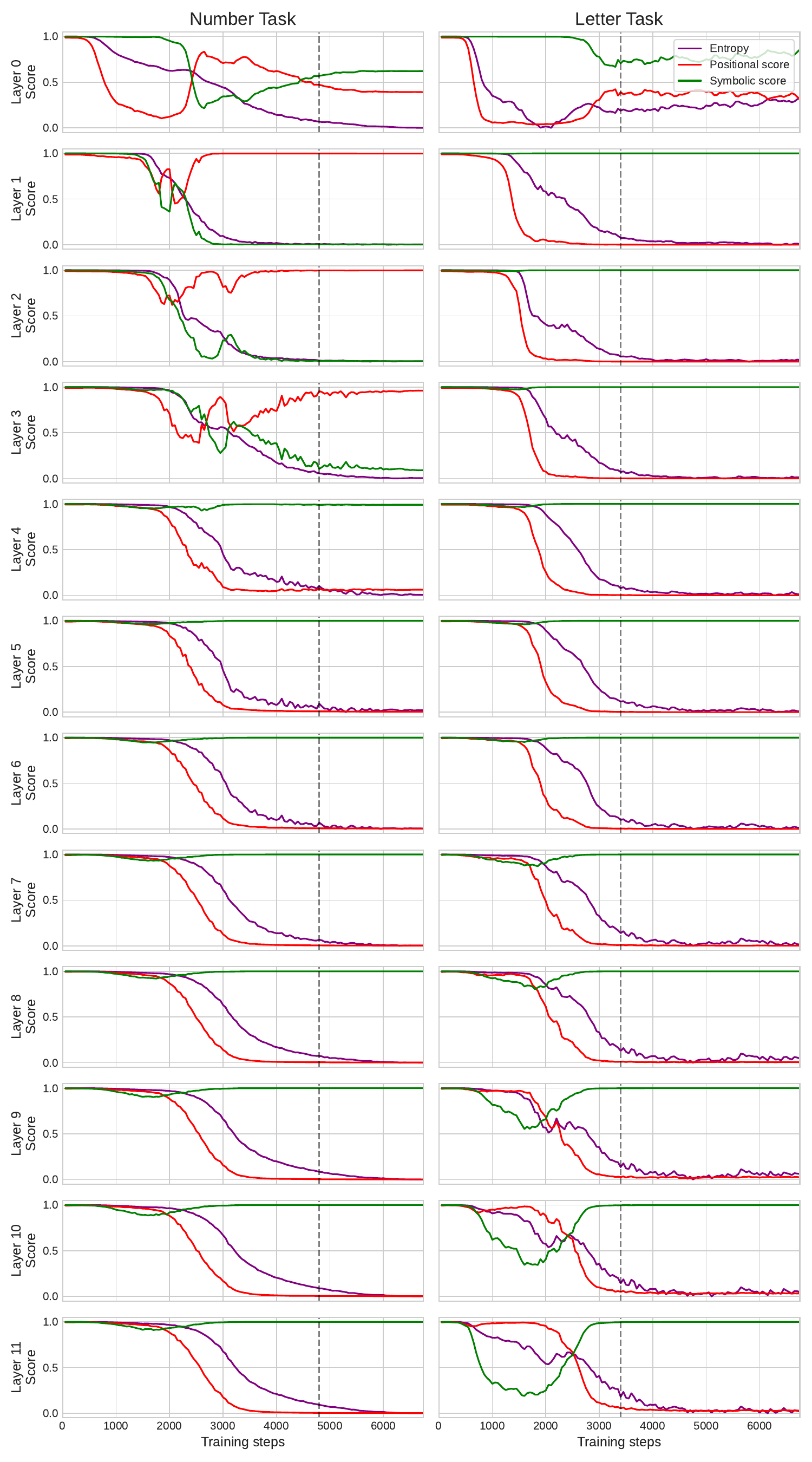}
    \caption{Positional and symbolic scores per layer for the models trained on the Number task (\textbf{left}) and the Letter task (\textbf{right}).}
    \label{fig:pos_sym_all_layers}
\end{figure}

\begin{figure}
    \centering
    \includegraphics[width=0.8\linewidth]{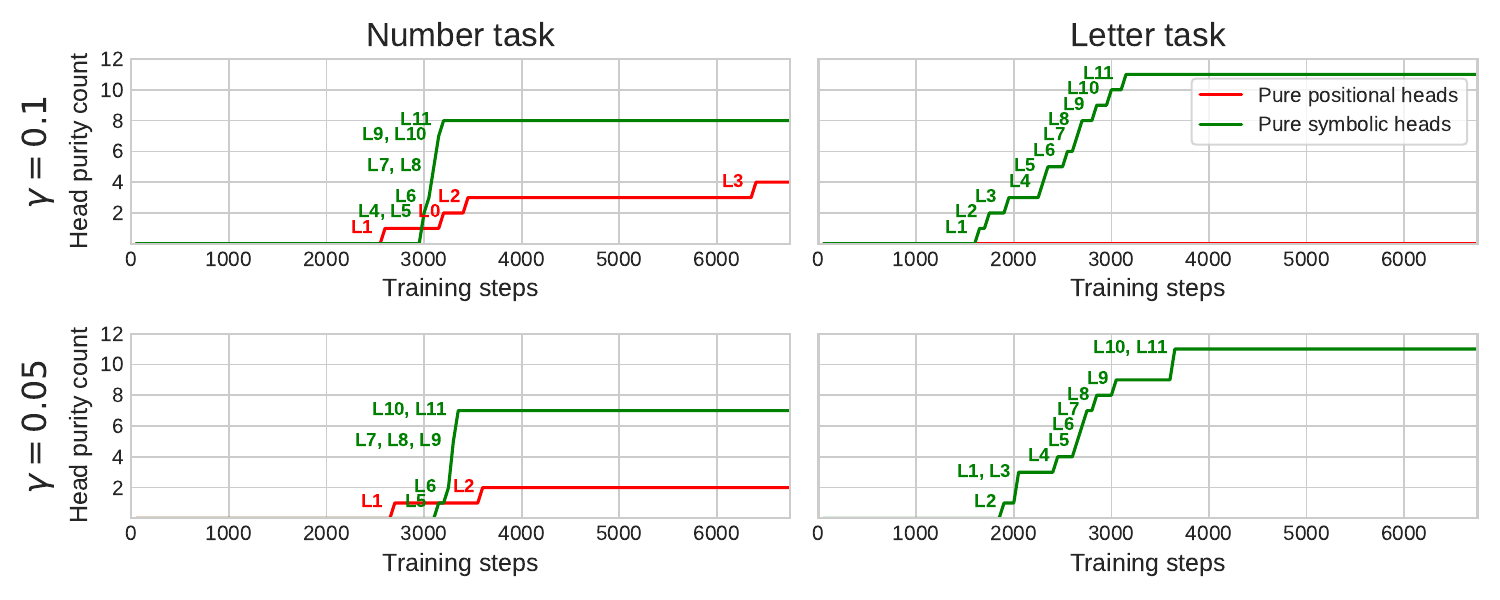}
    \caption{Dynamics of head purity counts for $\gamma \in \{0.1, 0.05\}$ (top and bottom, respectively)}
    \label{fig:gammas}
\end{figure}

\clearpage
\section{Formal definition of the basic functions} \label{sec:functions-pseudocode}

\subsection{Simple function definitions}
\begin{definition}[Relative Indexing function]
Let $m>0$ be the number of integers, let $\textsc{Alph}$ be a finite set of non-integer symbols, let $\textsc{Int} = \{1,...,m\}= [m]$, and define $\Sigma = \textsc{Alph} \cup \textsc{Int}$. For an input sequence $\bar{\sigma} = (\sigma_1, \sigma_2, \ldots, \sigma_n) \in \Sigma^n$ of length $n > m$, we define the function $\textcolor{black}{\mathrm{Index}}(\bar{\sigma}) \in \Sigma^n$ by specifying its output at position $j\leq n$.
If $\sigma_j = i \in \textsc{Int}$, then:
\[
\textcolor{tabred}{\mathrm{Index}}(\bar{\sigma})_j =
\begin{cases}
\sigma_{j-i} & \text{if } \sigma_{j-i} \in \textsc{Alph}, \\
\sigma_j & \text{otherwise}.
\end{cases}
\]
If instead $\sigma_j \in \textsc{Alph}$, then:
\[
\textcolor{gray}{\mathrm{Index}}(\bar{\sigma})_j = \sigma_j.
\]
\end{definition}

\begin{definition}[Retrieval function]\label{def:retrieval}
Let $m>0$ be the number of integers, let $\Sigma_{\mathrm{L}} = \textsc{Alph}$ be a set of non-integer symbols, and define $\Sigma_{\mathrm{R}} = \textsc{Alph} \cup Int$ where $Int = [m]$, let $\Sigma = \Sigma_{\mathrm{L}} \times \Sigma_{\mathrm{R}}$ be a finite set of bi-grams. Each token $\sigma \in \Sigma$ has the form $\sigma = \sigma^{\mathrm{L}} \sigma^{\mathrm{R}}$. A sequence $\bar{\sigma} = (\sigma_1, \sigma_2, \ldots, \sigma_n) \in \Sigma^n$ is called \emph{proper} if, for every $j \leq n$, whenever $\sigma_j^{\mathrm{R}} \in \textsc{Alph}$, there exists a position $k_j \leq j$ such that $\sigma_{k_j}^{\mathrm{L}} = \sigma_j^{\mathrm{R}}$, and if exists other position $k\leq j$ with the same property, then $\sigma_{k} = \sigma_{k_j}$. For a proper input $\bar{\sigma}$, we define the function $\textcolor{black}{\mathrm{Retrieval}}(\bar{\sigma}) \in \Sigma^n$ by specifying its output at position $j \leq n$. If $\sigma_j \in \textsc{Alph} \times \textsc{Alph}$, then:
\[
\textcolor{tabgreen}{\mathrm{Retrieval}}(\bar{\sigma})_j = \sigma_{k_j}
\]
If instead $\sigma_j \in \textsc{Alph} \times \textsc{Int}$, then:
\[
\textcolor{gray}{\mathrm{Retrieval}}(\bar{\sigma})_j = \sigma_j.
\]
\end{definition}

\begin{definition}[Reflexive function]
Let $\Sigma$ be a finite set. For any input sequence $\bar{\sigma} = (\sigma_1, \sigma_2, \ldots, \sigma_n) \in \Sigma^n$ of length $n$, we define the function $\textcolor{black}{\mathrm{Reflexive}}(\bar{\sigma}) \in \Sigma^n$ as follows:
\[
\textcolor{gray}{\mathrm{Reflexive}}(\bar{\sigma}) = \bar{\sigma}.
\]
\end{definition}

\begin{comment}
\begin{lstlisting}[
  caption={Pseudocode of Index, Retrieval, and Reflexive functions.}, 
  captionpos=b,
  label={fig:functions-pseudocode},
  language=Python,
  basicstyle=\ttfamily\small,
  keywordstyle=\color{blue},
  commentstyle=\color{gray},
  stringstyle=\color{teal},
  numberstyle=\tiny,
  frame=single,
  breaklines=false
]
def index(seq):
    out = seq.copy()
    for n in range(len(seq)):
        if is_number(seq[n]):
            i = seq[n]
            if is_letter(seq[n-i]):
                out[n] = seq[n-i]
    return out

def retrieval(seq):
    out = seq.copy()
    for n in range(len(seq)):
        if is_letter(seq[n]):
            for i in range(n):
                if seq[i][0] == seq[n][1]:
                    out[n] = seq[i]
                    break
    return out

def reflexive(seq): return seq
\end{lstlisting}\label{listing:functions}
\end{comment}

\subsection{Proofs of Theorem \ref{teo:mechanisms} and Theorem \ref{teo:discrepancy}}
\label{proof:teo}

\subsubsection{Selective Index function}

\begin{comment}
\begin{figure*}[ht!]
    \centering
    \includegraphics[width=1.0\linewidth]{figures/index_task.pdf}
    \caption{Illustration of the Selective  index function mechanism.}
    \label{fig:ill_index:_task}
\end{figure*}
\end{comment}

\paragraph{Requirements.} Let $m\geq 1$ be the number of integers, let $\textsc{Int} = \{1,...,m\} = [m]$, let $\textsc{Alph}$ be finite vocabulary of alphabetic symbols, let $E_\Sigma$ be a one-hot embedding over $\Sigma = \textsc{Alph} + \textsc{Int}$. Without loss of generality, set $E_\Sigma^{-1}(\sigma) = \sigma$ for $\sigma \in \textsc{Int}$. Consider $E: \Sigma \rightarrow \mathbb{R}^d$ be a embedding table, where $d = |\Sigma| + 1$ is the embedding dimension, such that 
$$
\forall \sigma\in \Sigma, \quad E(\sigma) = (E_\Sigma(\sigma), 1).
$$
\paragraph{Query and Key matrices.} We are going to define matrices $W^Q, W^K\in \mathbb{R}^{2\times d}$. Let $v$ be a non-zero $2$-dimensional column vector (to define). Let $\theta$ some angle (to define). Given that our embedding are one-hot, except the last coordinate (bias), each column of $W^Q$ and $W^K$ is determined by symbols in the vocabulary plus bias. Indeed, for $\sigma \in \Sigma$,
$$
W^QE(\sigma) = \begin{cases}
    v & \sigma \in \textsc{Alph} \\
    \rho({\theta})^{-\sigma}v & \sigma \in \textsc{Int}
\end{cases} \text{ and } \quad W^KE(\sigma) = v
$$
where $\rho(\theta)$ is rotatory matrix with angle $\theta$.

\paragraph{Dot-product with RoPE.} For our construction, we will require a RoPE with only one angle $\theta_\text{RoPE}$, that we will take to be $\theta_\text{RoPE}=\theta$. For $\sigma_\textrm{query}, \sigma_\textrm{key} \in \Sigma$ and positions $n$ and $j \in\{1,...,n\}$, the logit function is computed as follows
$$
(W^KE(\sigma_\textrm{key}))^\top \rho(\theta)^{n-j} W^QE(\sigma_\textrm{query}) = \begin{cases}
    v^\top \rho({\theta})^{n-j} v & \sigma_\textrm{query} \in \textsc{Alph} \\
    v^\top \rho(\theta)^{n-j}\rho(\theta)^{-\sigma_\textrm{query}}v & \sigma_\textrm{query} \in \textsc{Int}
\end{cases}
$$
Note that for $\sigma \in \textsc{Int}$,
$$
v^\top \rho(\theta)^{n-j}\rho(\theta)^{-\sigma}v = v^\top \rho(\theta)^{n-j-\sigma}v = |v|^2\cos((n-j-\sigma)\theta).
$$
In particular, for any $v\neq0$, the previous quantity is maximal at $j=n-\sigma$. 
\paragraph{Logit computation} Let $\bar{\sigma}\in \Sigma^n$ be an input of length $n>1$. For simplicity, consider $v=(0,1)$, then the logit function results into 
$$
L_{\theta}(E(\sigma_n), n, E(\sigma_j), j) = \begin{cases}
    \cos((n-j)\theta) & \sigma_n \in \textsc{Alph} \\
    \cos((n-j-\sigma_n)\theta) & \sigma_n \in \textsc{Int}
\end{cases}
$$
Notice that, for all $\theta\neq0$, if $\sigma_n \in \textsc{Alph}$, position $n$ maximize the logit, while position $n-\sigma_n$ maximize the logit when $\sigma_n \in \textsc{Int}$. In both cases the logit equals to $1$. However, in order for the construction to work, we also need this maximum to be unique. 

\paragraph{Discrepancy computation.} Let us define the discrepancy for an input $\bar{\sigma}$ as the difference between the largest logit value and any other:
$$
\Delta_\theta(\bar{\sigma}) = 1 - \begin{cases}
 \max \lbrace \cos((n-j)\theta): 1 \leq j < n\rbrace  & \sigma_n \in \textsc{Alph} \\
\max \lbrace \cos((n-j -\sigma_n)\theta): 1 \leq j \leq n, j \neq n-\sigma_n \rbrace & \sigma_n \in \textsc{Int} \\
\end{cases}
$$
Note that this quantity also controls the robustness of the mechanism to perturbation or errors in the numerical computations (e.g. due to round-off). In particular, if $\theta \notin \mathbb{Q}$ then 
$$
\forall n> 1, \quad \Delta_{\theta}(\bar{\sigma})>0,
$$
this is because $\cos(t\theta)$ is maximal at $t \in \frac{2\pi}{\theta}\mathbb{Z}$. 

Let us define the discrepancy for inputs of length $n$ as follows
$$
\Delta_{\theta}(n) := \max_{\bar{\sigma}} \Delta_\theta(\bar{\sigma}).
$$
Notice that this discrepancy (for any input $\bar{\sigma}$) is upper-bounded by
$$
\Delta_{\theta}( n) \leq 1 - \cos(\theta)
$$
which highlight the fact that $\theta$ can not be too small. 

On the other hand, the discrepancy (for any input $\bar{\sigma}$) is also affected by the length $n$ of the input sequence. Indeed, we have
$$
\Delta_{\theta}(n)\leq 2 \frac{\pi^2}{n^2}.
$$
The proof of this result is given as follows. Consider $n$ sub-intervals of length $2\pi/n$ for the interval $[0, 2\pi)$. Let $\Theta_n = \{t\theta \mod 2\pi: t\in\{0,1,...,n\}\}$ be a set of $n+1$ angles. Given that we have $n+1$ angles and $n$ sub-intervals, then it should be exist two different angles $\theta_i, \theta_j \in \Theta_n$ inside the same sub-interval. In this way, using $i-j = t \in \{1,...,n\}$, we have that $|t\theta \mod 2\pi| \leq 2\pi/n$. We conclude using the fact that $1-\cos(x) \leq x^2/2$ for $x\in \mathbb{R}$.

From now we consider $\theta$ be irrational. In this way, for input of length at least $n>1$, the logit function has an unique maximum.

\paragraph{Residual Stream and Attention Output.} We consider an additive residual stream of the form
\[
F(x,a) = x + a ,
\]
where $x \in \mathbb{R}^d$ is the input representation and $a \in \mathbb{R}^d$ is the attention output.  
The attention output is a weighted sum of value vectors. To formalize this, we introduce a value  matrix $W^V \in \mathbb{R}^{d \times d}$. We define $W^V$ to be diagonal. For each coordinate $c \in \{1,\dots,d\}$,
\[
\left( W^V \right)_{cc} =
\begin{cases}
\alpha>1 & \exists\, \sigma \in \textsc{Alph} \text{ such that } E_\Sigma(\sigma)_c = 1, \\
0 & \text{otherwise}.
\end{cases}
\]
From now we set $\alpha = 2$.
\paragraph{Average Hard Attention.} Let $\sigma_{j(\bar{\sigma})}$ denote the most attended token. For our construction this position is unique. Then the average hard attention act as unique hard attention. Under unique hard attention, the attention output is
\[
a_n(\bar{\sigma}) = W^V E(\sigma_{j(\bar{\sigma})}).
\]
The residual stream becomes
\[
R(\sigma_n, \sigma_{j(\bar{\sigma})}) := F(E(\sigma_n), a_n(\bar{\sigma}))
= E(\sigma_n) + W^V E(\sigma_{j(\bar{\sigma})}).
\]
\paragraph{Unembedding Step.} Let $\sigma \in \Sigma$ be a candidate output symbol. The output logit associated with $\sigma$ is
\[
L_{\mathrm{Out}}(\sigma) = F(E(\sigma_n), a_n(\bar{\sigma}))^\top (E_\Sigma(\sigma), 0).
\]
Since $E_\Sigma(\sigma)$ is one-hot, maximizing $L_{\mathrm{Out}}(\sigma)$ is equivalent to requiring that the coordinate corresponding to $\sigma$ is strictly larger than all others.
We now show that, under average hard attention, the correct coordinate is uniquely maximized, with gap at least $1$.
\subparagraph{Case 1: $\sigma_n \in \textsc{Alph}$} In this case, the most attended token satisfies $j(\bar{\sigma}) = n$, hence $\sigma_{j(\bar{\sigma})} = \sigma_n \in \textsc{Alph}$. In this cases, the residual stream equals to three times the input embedding:
$$
R(\sigma_n, \sigma_n) = 3 E(\sigma_n).
$$
Then the correct coordinate is uniquely maximized and with gap equal to $3$.
\subparagraph{Case 2: $\sigma_n \in \textsc{Int}$ and $\sigma_{j(\bar{\sigma})} \in \textsc{Alph}$} Let $c^* = E_\Sigma^{-1}(\sigma_{j(\bar{\sigma})})$. Then
\[
R(\sigma_n, \sigma_{j(\bar{\sigma})})_{c^*} = 0+2,
\]
while the largest competing coordinate $c = E_\Sigma^{-1}(\sigma_n)$ satisfies
\[
R(\sigma_n, \sigma_{j(\bar{\sigma})})_c = 1+0.
\]
Hence the correct coordinate is maximized with gap at least $1$.
\subparagraph{Case 3: $\sigma_n \in \textsc{Int}$ and $\sigma_{j(\bar{\sigma})} \in \textsc{Int}$} In this case, the value projection vanishes and the residual stream equals the input embedding:
\[
R(\sigma_n, \sigma_{j(\bar{\sigma})}) = E(\sigma_n).
\]
Again, the correct coordinate is uniquely maximized and with gap equal to $1$.

\paragraph{Conclusion, Average Hard Attention and Softmax.} In all cases, the output logit $L_{\mathrm{Out}}(\sigma)$ is maximized at the correct symbol $\sigma$, and the minimum decoding gap $\delta$ between coordinates is at least $1$. Since the decoding gap is bounded below by $\delta \ge 1$, and the minimum discrepancy induced by irrational angle $\theta$ is strictly positive, Lemma~\ref{lemma:scaled_logit} implies that the same conclusion holds, for  inputs up to a certain length, when attention is computed using softmax instead of average hard attention.

\begin{lemma}[From Average Hard Attention to Softmax Attention] \label{lemma:scaled_logit} Let $\Sigma$ be a finite vocabulary and let $n_0 > 1$ be a fixed maximum input length.  
Let $E:\Sigma \to \mathbb{R}^d$ be an embedding function and let
\[
L:\mathbb{R}^d \times \mathbb{N} \times \mathbb{R}^d \times \mathbb{N} \to \mathbb{R}
\]
be a logit function. For an input $\bar{\sigma} = (\sigma_1,\dots,\sigma_n) \in \Sigma^n$, $n \le n_0$, define
\[
\mu_L(\bar{\sigma}) = \max_{1 \le j \le n} L(E(\sigma_n), n, E(\sigma_j), j),
\]
and let
\[
J(\bar{\sigma}) = \left\{ j \le n : L(E(\sigma_n), n, E(\sigma_j), j) = \mu_L(\bar{\sigma}) \right\}.
\]
Let $W^V \in \mathbb{R}^{d \times d}$ and define the softmax attention weights
\[
w_n^j(\bar{\sigma}; \beta) =
\frac{\exp\bigl(\beta L(E(\sigma_n), n, E(\sigma_j), j)\bigr)}
{\sum_{i=1}^n \exp\bigl(\beta L(E(\sigma_n), n, E(\sigma_i), i)\bigr)},
\quad \beta > 0,
\]
with attention output
\[
a_n^\beta(\bar{\sigma}) = \sum_{j=1}^n w_n^j(\bar{\sigma}; \beta)\, W^V E(\sigma_j).
\]
Define the average hard-attention output
\[
a_n(\bar{\sigma})
= \frac{1}{|J(\bar{\sigma})|}
\sum_{j \in J(\bar{\sigma})} W^V E(\sigma_j).
\]
Let $F:\mathbb{R}^d \times \mathbb{R}^d \to \mathbb{R}^{|\Sigma|}$ be continuous, and define the decoding gap
\[
\delta :=
\min_{\substack{\bar{\sigma} \in \Sigma^n,\, n \le n_0}}
\ \min_{c \neq c^*(\bar{\sigma})}
\left(
F(E(\sigma_n), a_n(\bar{\sigma}))_{c^*(\bar{\sigma})}
-
F(E(\sigma_n), a_n(\bar{\sigma}))_{c}
\right)
\]
where
\[
c^*(\bar{\sigma}) := \arg\max_{1 \le c \le |\Sigma|}
F(E(\sigma_n), a_n(\bar{\sigma}))_c
\]
is assumed to be unique. If $\delta >0$, then there exists $\beta > 0$ such that for all $\bar{\sigma} \in \Sigma^n$, $n \le n_0$, and all $c \neq c^*(\bar{\sigma})$,
\[
F(E(\sigma_n), a_n^\beta(\bar{\sigma}))_{c^*(\bar{\sigma})}
>
F(E(\sigma_n), a_n^\beta(\bar{\sigma}))_c.
\]
In particular,
\[
\arg\max_{1 \le c \le |\Sigma|} F(E(\bar{\sigma}), a_n^\beta(\bar{\sigma}))_c
=
\arg\max_{1 \le c \le |\Sigma|} F(E(\bar{\sigma}), a_n(\bar{\sigma}))_c.
\]
\end{lemma}

\begin{proof} Fix an input $\bar{\sigma} \in \Sigma^n$, $n \le n_0$. Define the discrepancy for $\bar{\sigma}$ by
\[
\Delta(\bar{\sigma}) =
\begin{cases}
\mu_L(\bar{\sigma}) - \max\{ L(E(\sigma_n), n, E(\sigma_j), j) : j \notin J(\bar{\sigma}) \}
& |J(\bar{\sigma})| < n, \\
0 & |J(\bar{\sigma})| = n.
\end{cases}
\]
For each $n \le n_0$, define the minimum discrepancy for input length $n$ as
\[
\lambda(n) =
\begin{cases}
\min\{ \Delta(\bar{\sigma}) : \bar{\sigma} \in \Sigma^n,\ \Delta(\bar{\sigma}) > 0 \}
& \text{if such } \bar{\sigma} \text{ exists}, \\
0 & \text{otherwise}.
\end{cases}
\]
Define the total discrepancy as follows
\[
\lambda_{\leq n_0} := \min_{1 < n \le n_0} \lambda(n)
\]
Notice that if $|J(\bar{\sigma})| = n$, then all logits are equal and
$w_n^j(\beta) = 1/n$ for all $\beta$, so
$a_n^\beta(\bar{\sigma}) = a_n(\bar{\sigma})$ and the claim holds trivially.
Hence assume $|J(\bar{\sigma})| = k < n$.
For $j \in J(\bar{\sigma})$, we have
\[
L(E(\sigma_n), n, E(\sigma_j), j) = \mu_L(\bar{\sigma}),
\]
while for $j \notin J(\bar{\sigma})$,
\[
L(E(\sigma_n), n, E(\sigma_j), j)
\le \mu_L(\bar{\sigma}) - \lambda_{\leq n_0}.
\]
Therefore, for $\beta \geq 1$,
\[
k e^{\mu_L(\bar{\sigma})} \leq \sum_{i=1}^n e^{\beta L(E(\sigma_n), n, E(\sigma_i), i)}
\le
k e^{\beta \mu_L(\bar{\sigma})}
+ (n-k) e^{\beta(\mu_L(\bar{\sigma})-\lambda_{\leq n_0})}.
\]
This yields, for $j \in J(\bar{\sigma})$,
\[
\frac{1}{k}
\geq
w_n^j(\bar{\sigma};\beta)
\geq
\frac{1}{k + (n-k)e^{-\beta\lambda_{\leq n_0}}} \rightarrow \frac{1}{k} 
\quad \text{as } \beta \rightarrow \infty.
\]
On the other hand, for $j \notin J(\bar{\sigma})$, 
\[
0\leq w_n^j(\bar{\sigma};\beta)
\le
\frac{e^{-\beta(\mu_L(\bar{\sigma}) - L(E(\sigma_n), n, E(\sigma_j), j))}}{k} \rightarrow 0 
\quad \text{as } \beta \rightarrow \infty.
\]
Hence,
\[
a_n^\beta(\bar{\sigma})
\rightarrow
\frac{1}{k} \sum_{j \in J(\bar{\sigma})} W^V E(\sigma_j)
= a_n(\bar{\sigma})
\quad \text{as } \beta \rightarrow \infty.
\]
Since $\Sigma$ is finite and $n_0 < \infty$, the set of all inputs
$\bigcup_{n \le n_0} \Sigma^n$ is finite. Therefore, the above convergence
is uniform over all such inputs.
By continuity of $F$, for any $\epsilon > 0$ there exists $\beta\geq1$,
\[
\max_{c}
\left|
F(E(\sigma_n), a_n^\beta(\bar{\sigma}))_c
-
F(E(\sigma_n), a_n(\bar{\sigma}))_c
\right|
\le \epsilon
\]
uniformly over all $\bar{\sigma}$ with $n \le n_0$. Choose $\epsilon \in (0,\delta/2)$. Then for any $c \neq c^*(\bar{\sigma})$,
\begin{align*}
F(E(\sigma_n), a_n^\beta(\bar{\sigma}))_c
&\le
F(E(\sigma_n), a_n(\bar{\sigma}))_c + \epsilon \\
&\le
F(E(\sigma_n), a_n(\bar{\sigma}))_{c^*(\bar{\sigma})}
- \delta + \epsilon \\
&<
F(E(\sigma_n), a_n(\bar{\sigma}))_{c^*(\bar{\sigma})}
- \epsilon \\
&<
F(E(\sigma_n), a_n^\beta(\bar{\sigma}))_{c^*(\bar{\sigma})}.
\end{align*}
This proves the claim.
\end{proof}

\subsubsection{Retrieval function}

\paragraph{Requirements.} Let $E_\Sigma$ be a one-hot embedding over $\Sigma = \Sigma_\text{L} \times \Sigma_\text{R}$, where $\Sigma_\text{L} = \textsc{Alph}$ and $\Sigma_\text{R} = \textsc{Alph} \cup \textsc{Int}$. Let $\phi_\textsc{Alph}: \textsc{Alph} \rightarrow \{1,...,|\textsc{Alph}|\}$ be some enumeration of symbols from $\textsc{Alph}$ into numbers. Let $\phi: \textsc{Alph}\cup \textsc{Int} \rightarrow \{1,...,|\textsc{Alph}|, |\textsc{Alph}| + 1, ..., |\textsc{Alph}| + |\textsc{Int}|\}$ be an extension of $\phi_\textsc{Alph}$, such that $\phi(\sigma) = \phi_\textsc{Alph}(\sigma)$ for $\sigma \in \textsc{Alph}$, and $\phi(i) = |\textsc{Alph}| + i$ for $i\in \textsc{Int}$. Without loss of generality, set $E_\Sigma^{-1}$ be consistent with $\phi$, in the following sense, for $\sigma=\sigma^L\sigma^R \in \Sigma$,
$$
E_\Sigma^{-1}(\sigma) = (\phi(\sigma^\text{L}) - 1) |\Sigma_\text{R}| + \phi(\sigma^\text{R}).
$$
\paragraph{Query and Key matrices.} We are going to define matrices $W^Q, W^K\in \mathbb{R}^{2\times d}$. Let $v$ be a non-zero $2$-dimensional column vector (to define). Let $\omega$ some angle (to define). Given that our embedding are one-hot each column of $W^Q$ and $W^K$ is determined by symbols in the vocabulary. Indeed, for $\sigma \in \Sigma$,
$$
W^QE(\sigma) = \begin{cases}
    \rho({\omega})^{-\phi(\sigma^\text{R})}v & \sigma \in \textsc{Alph} \times \textsc{Alph} \\
    \rho({\omega})^{-\phi(\sigma^\text{L})}v & \text{otherwise}
\end{cases} \text{ and } \quad W^KE(\sigma) = \rho({\omega})^{-\phi(\sigma^\text{L})}v.
$$
where $\rho(\omega)$ is rotatory matrix with angle $\omega$.

\paragraph{Dot-product with RoPE.} For our construction, we are going to include RoPE in order to exhibit how this positional encoding induce length generalization failure. Let $\theta$ be one angle for RoPE (to define). For $\sigma_\textrm{query}, \sigma_\textrm{key} \in \Sigma$ and positions $n$ and $j \in\{1,...,n\}$, the logit function is computed as follows
$$
(W^KE(\sigma_\textrm{key})^\top  \rho(\theta)^{n-j}W^QE(\sigma_\textrm{query}) = \begin{cases}
    v^\top \rho({\omega})^{\phi(\sigma_\textrm{key}^\text{L})}  \rho(\theta)^{n-j}\rho({\omega})^{-\phi(\sigma_\textrm{query}^\text{R})}v & \sigma_\textrm{query} \in \textsc{Alph} \times \textsc{Alph} \\
    v^\top \rho({\omega})^{\phi(\sigma_\textrm{key}^\text{L})}  \rho(\theta)^{n-j} \rho({\omega})^{-\phi(\sigma_\textrm{query}^\text{L})}v & \text{otherwise}
\end{cases}
$$
\paragraph{Logit computation.} Let $\bar{\sigma} \in \Sigma^n$ be an input of length $n>1$. For simplicity consider $v=(0,1)$, then the logit function results into
$$
L_{\theta}(E(\sigma_n), n, E(\sigma_j), j) = \begin{cases}
    \cos( (\phi(\sigma_j^\text{L})-\phi(\sigma_n^\text{R}))\omega +(n-j)\theta ) & \sigma_n \in \textsc{Alph} \times \textsc{Alph} \\
    \cos((\phi(\sigma_j^\text{L})-\phi(\sigma_n^\text{L}))\omega +(n-j)\theta)  & \text{otherwise}
\end{cases}
$$
Let $n_0 > 1$. Notice that for $0<\theta < \omega/2 n_0$, we have that for $n\leq n_0$ if exists $\sigma_j$ with left-letter equal to the query letter ($\sigma_n^L$ or $\sigma_n^R$ depending on the input) at certain position $j=j(\bar{\sigma})$, then
$$
\mu_L(\bar{\sigma}) = \max_{1\leq j\leq n} L_\theta(E(\sigma_n, n, E(\sigma_j), j) =  \cos((n-j(\bar{\sigma}))\theta).
$$

Notice that \textit{(i)} if two (key) symbols in the sequence $\sigma_{j_1}$ and $\sigma_{j_2}$, at different positions, both have the maximum logit, then this two symbols must have the same left-symbol $\sigma_{j_1}^L = \sigma_{j_2}^L$. Then, from the assumption that $\bar{\sigma}$ is a proper input (see Definition \ref{def:retrieval}), this implies that $\sigma_{j_1} = \sigma_{j_2}$ are the same symbol. On the other hand, \textit{(ii)} if all the symbol $\sigma_j$, with $j\leq n$, has the same left-symbol $\sigma_j^L$, then all the sequence is attended equally. Then by the property \textit{(i)} the input sequence is a constant sequence. In any other case, the logit is not constant. For the discrepancy computation, we are going to assume that the input sequence is not constant. 

\paragraph{Discrepancy computation.} The discrepancy for this mechanism is give by
$$
\Delta_{\theta}(\bar{\sigma}) = \cos((n-j(\bar{\sigma}))\theta) - \max_{\substack{1 < k \leq |\textsc{Alph}| \\ k \neq \phi(\sigma_{j(\bar{\sigma})}) \\ 1 \leq j \leq n}} \cos((k-1)\omega + (n-j)\theta) \leq 1-\cos(\omega - n\theta).
$$
In this way, the discrepancy is upper-bounded by
$$
\Delta_\theta(n) \leq 1-\cos(\omega - n\theta).
$$

\paragraph{Residual Stream and Attention Output.} We consider a attention-only residual stream of the form
$$
F(x, a) = a
$$
where $x \in \mathbb{R}^d$ is the input representation and $a\in \mathbb{R}^d$ is the attention output. Note that the construction could also be done with a non-zero residual stream, but this would only make the proof more technical without affecting the discrepancy analysis, so we opt for keeping it as simple as possible.

The attention output is a weighted sum of value vectors. For our construction, we define the value matrix $W^V \in \mathbb{R}^{d\times d}$ to be the $d\times d$ identity function. 
\paragraph{Hard Attention} Let $\sigma_{j(\bar{\sigma})}$ denote any most attended token. If are multiple candidates, then by construction all the most attended tokens must be associated to the same symbol. Under average hard attention, the attention output is
$$
a_n(\bar{\sigma}) = W^VE(\sigma_{j(\bar{\sigma})}) = E(\sigma_{j(\bar{\sigma})}).
$$
The residual stream becomes
$$
R(\sigma_n, \sigma_{j(\bar{\sigma})}) := F(E(\sigma_n), a_n(\bar{\sigma})) = E(\sigma_{j(\bar{\sigma})})
$$
\paragraph{Unembedding Step.} Let $\sigma \in \Sigma$ be a candidate output symbol. The output logit associated with $\sigma$ is
$$
L_\text{Out}(\sigma) = F(E(\sigma_n), a_n(\bar{\sigma}))^\top E_\Sigma(\sigma).
$$
Since $E_\Sigma(\sigma)$ is one-hot, maximizing $L_\text{Out}(\sigma)$ is equivalent to requiring that the coordinate corresponding to $\sigma$ is strictly larger than all others. We now show that, under hard attention, the correct coordinate is uniquely maximized, with gap at least $1$.
\subparagraph{Case 1: $\sigma_n \in \textsc{Alph} \times \textsc{Int}$} In this case, the most attended token satisfies $j(\bar{\sigma}) = n$, hence $\sigma_{j(\bar{\sigma})} = \sigma_n \in \textsc{Alph}$. In this cases, the residual stream equals the input embedding:
$$
R(\sigma_n, \sigma_n) = E(\sigma_n).
$$
Then the correct coordinate is uniquely maximized and with gap equal to $1$.
\subparagraph{Case 2: $\sigma_n \in \textsc{Alph} \times \textsc{Alph}$} In this case, the residual stream equals the attention output. Again, the correct coordinate is uniquely maximized at $c^*=E_\Sigma^{-1}(\sigma_{j(\bar{\sigma})})$ and with gap at least $1$. 
\paragraph{Conclusion, Average Hard Attention and Softmax} In both cases, the output logit $L_{\mathrm{Out}}(\sigma)$ is maximized at the correct symbol $\sigma$, and the minimum decoding gap $\delta$ between coordinates is at least $1$. Since the decoding gap is bounded below by $\delta \ge 1$, and the minimum discrepancy (on non-constant inputs) induced by angle $\theta=2\pi/|\textsc{Alph}|$ is strictly positive, Lemma~\ref{lemma:scaled_logit} implies that the same conclusion holds, for bounded inputs, when attention is computed using softmax instead of average hard attention.

\subsubsection{Proof of Remark \ref{rem:impossibility}}\label{app:proof_rem}

We start by showing that a purely symbolic head cannot compute the Selective Index function. The proof is similar to that of Theorem 2 from \cite{urrutia2025decoupling}. Let $T$ be a one-layer model with a purely symbolic attention head. By Lemma 1 from \cite{urrutia2025decoupling}, it holds that $T$ is invariant under permutations of tokens in the input. Then, by swapping two letters in the target zone of an input chosen such that the obtained input yields a different answer, we obtain that, since the model outputs the same answer for both, must give an incorrect answer in one of them. The proof for the Retrieval function follows from Theorem 4 in \cite{urrutia2025decoupling} by noticing that our Retrieval function can be expressed as a sequence-to-sequence version of their Information Retrieval Task.  

\newpage
\section{Generalization}
\label{app:generalization}

\begin{figure*}[htbp!]
    \centering
    \includegraphics[width=0.8\linewidth]{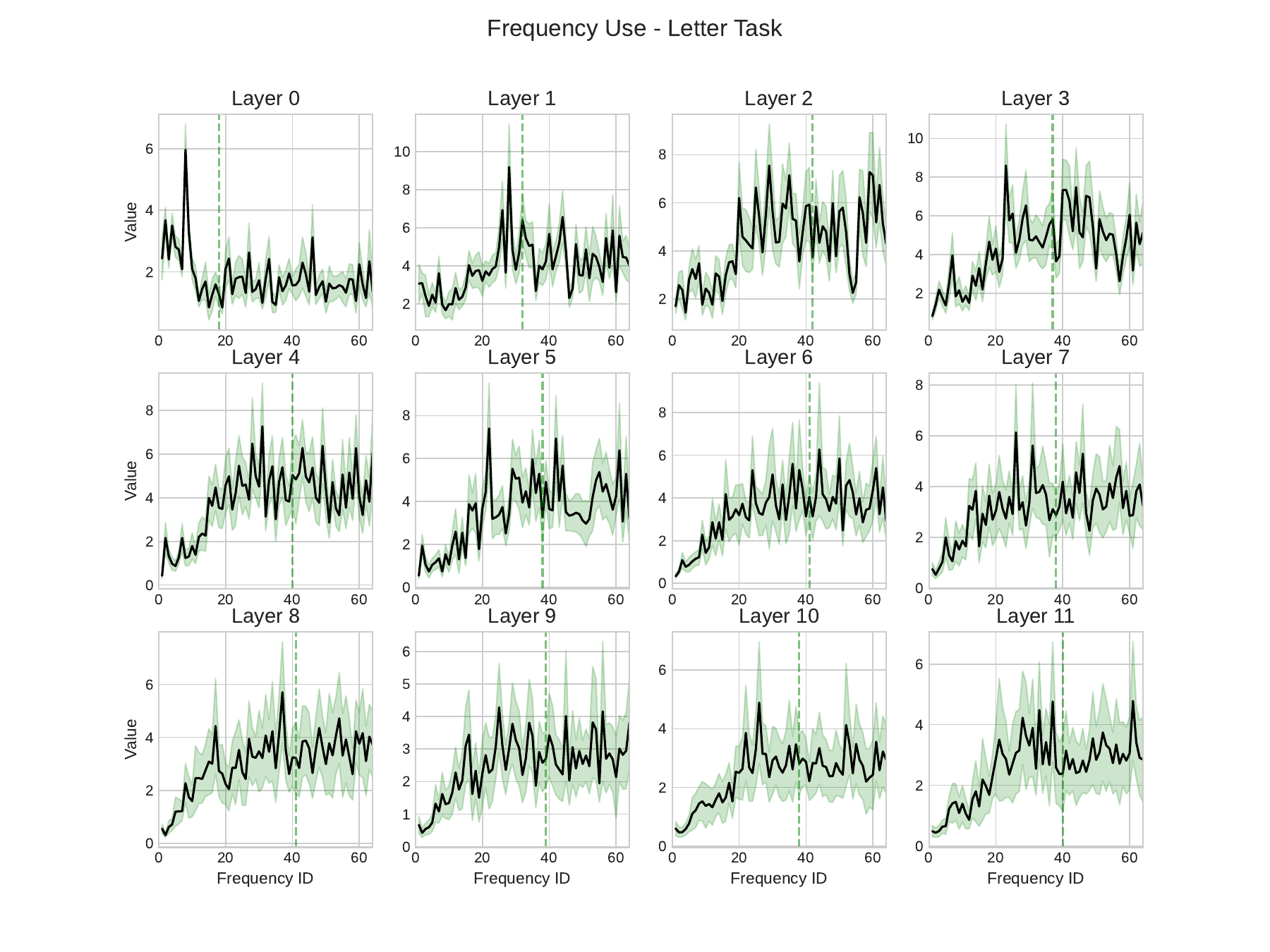}
    \includegraphics[width=0.8\linewidth]{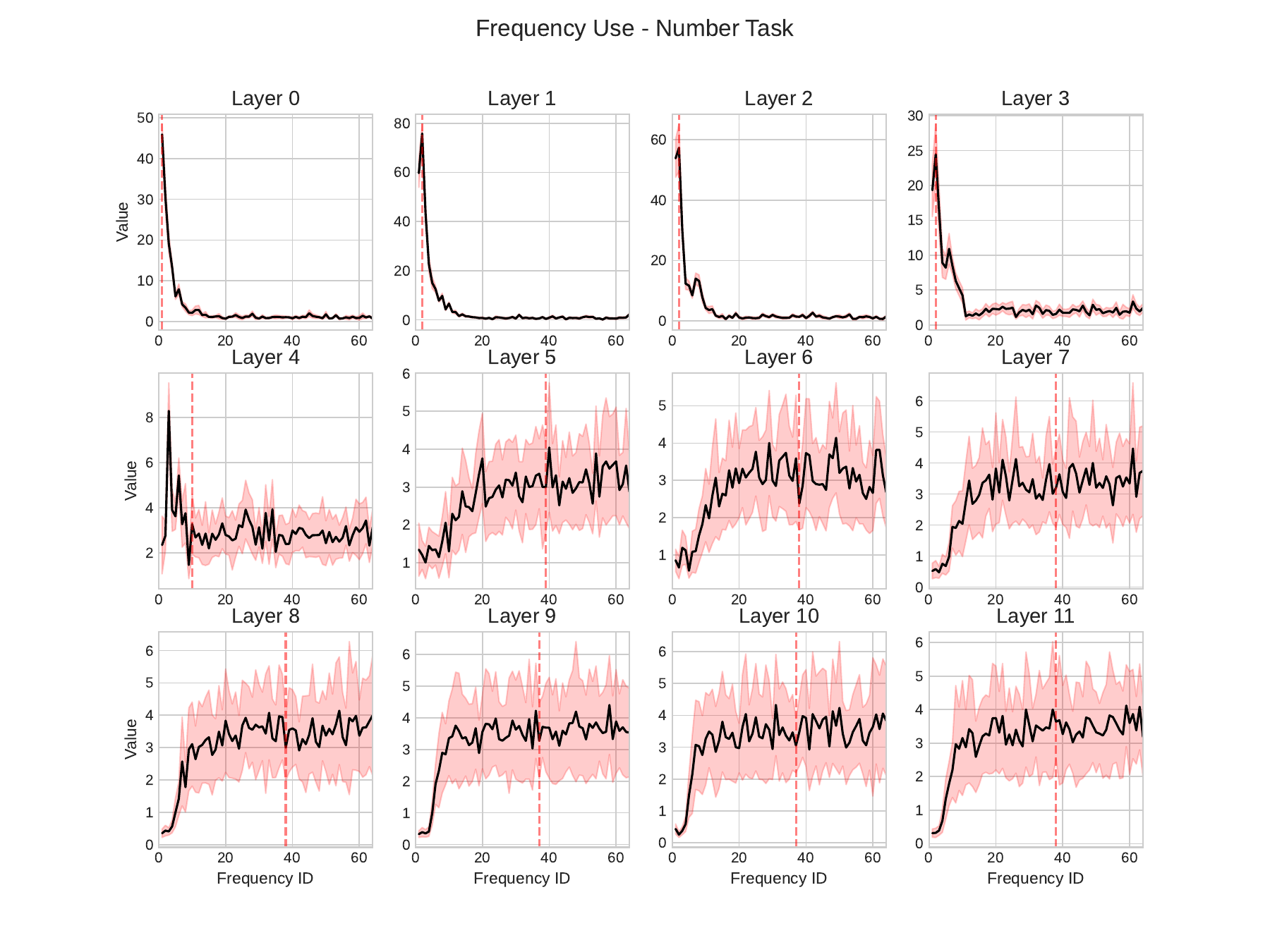}
    \caption{Frequency Use for the Letter and Number task.}
    \label{fig:angles}
\end{figure*}

\begin{table}[]
\centering
\begin{tabular}{|r|rr|rr|}
\hline
\multicolumn{1}{|l|}{} & \multicolumn{2}{l|}{Number Task}              & \multicolumn{2}{l|}{Letter Task}              \\ \hline
Layer                  & \multicolumn{1}{r|}{Frequency ID} & Frequency & \multicolumn{1}{r|}{Frequency ID} & Frequency \\ \hline
0                      & \multicolumn{1}{r|}{1}            & 1.0000    & \multicolumn{1}{r|}{18}           & 0.0866    \\ \hline
1                      & \multicolumn{1}{r|}{2}            & 0.8660    & \multicolumn{1}{r|}{32}           & 0.0115    \\ \hline
2                      & \multicolumn{1}{r|}{2}            & 0.8660    & \multicolumn{1}{r|}{42}             & 0.0027    \\ \hline
3                      & \multicolumn{1}{r|}{2}            & 0.8660    & \multicolumn{1}{r|}{37}           & 0.0056    \\ \hline
4                      & \multicolumn{1}{r|}{10}           & 0.2738    & \multicolumn{1}{r|}{40}           & 0.0037    \\ \hline
5                      & \multicolumn{1}{r|}{39}           & 0.0042    & \multicolumn{1}{r|}{38}           & 0.0049    \\ \hline
6                      & \multicolumn{1}{r|}{38}           & 0.0049    & \multicolumn{1}{r|}{41}           & 0.0032    \\ \hline
7                      & \multicolumn{1}{r|}{38}           & 0.0049    & \multicolumn{1}{r|}{38}           & 0.0049    \\ \hline
8                      & \multicolumn{1}{r|}{38}           & 0.0049    & \multicolumn{1}{r|}{41}           & 0.0032    \\ \hline
9                      & \multicolumn{1}{r|}{37}           & 0.0056    & \multicolumn{1}{r|}{39}           & 0.0042    \\ \hline
10                     & \multicolumn{1}{r|}{37}           & 0.0056    & \multicolumn{1}{r|}{38}           & 0.0049    \\ \hline
11                     & \multicolumn{1}{r|}{38}           & 0.0049    & \multicolumn{1}{r|}{40}           & 0.0037    \\ \hline
\end{tabular}
\caption{Weighted averaged ID's and the corresponding frequencies.}\label{table:angles}
\end{table}

\section{Generalization on LLMs}
\label{app:generalization_llm}

To test generalization in frontier LLMs, we constructed prompts for the 1-hop number task and the 1-hop letter task, each containing natural language instructions, worked examples, and a final query (see Table~\ref{table:prompts}). The prompts explicitly required a single-token answer, and all models were evaluated with reasoning disabled. We then applied two screening criteria. First, since the tasks are non-trivial, we required at least 85\% accuracy on the 1-hop letter task at sequence length 16, ensuring that the base task could be performed before generalization was tested; this criterion ruled out Qwen 3.6 Plus, GPT 4.1, and Deepseek v4 Pro. Second, we required that the model generate a single token rather than a chain-of-thought at longer sequence lengths, which ruled out the more recent frontier Claude models, including Opus 4.6 and Opus 4.7. Three models passed both filters: OpenAI's GPT 5.4 and GPT 5.5, and Anthropic's Claude Sonnet 3.7.

We ran the experiment five times, generating a fresh dataset for each run using the run number as the random seed. Each dataset contained sequences with target windows of sizes \{12, 16, 24, 32, 48, 64, 80, 100, 150, 200, 300\} and a single query token $Q$ appended at the end. Given $Q$, the model's task was to retrieve the corresponding token from the target window: in the 1-hop number task, $Q$ is an integer indicating how many positions to count back from the end; in the 1-hop letter task, $Q=(Q_1,Q_2)$ is a pair, and the answer is the target window token whose first component matches $Q_2$ We aggregated results across the five runs, reporting the mean and standard deviation of accuracy for each model at each sequence length (see Table \ref{table:generalization_llm}).

% Please add the following required packages to your document preamble:
% \usepackage{booktabs}
\begin{table}[]
\centering
\begin{tabular}{@{}lll@{}}
\toprule
Task   & Prompt                                                                                                                                                                                                                                                                                                                                                                                                                                                                                                                                                                                                                                                                                                                                                                                                     &  \\ \midrule
Number & \begin{tabular}[c]{@{}l@{}}You will be given a sequence of space-separated tokens. Your task is to find the target \\ token by following these rules:\\ \\ 1. Identify the very last token in the sequence, which will always be a number `n`.\\ 2. Move `n` positions to the left from that final token.\\ 3. The token you land on is the target token.\\ \\ Output just the target token and absolutely nothing else. Do not include any \\ explanations or formatting.\\ \\ Example 1:\\ Input: "a z b y c x d w 5"\\ Output: y\\ \\ Example 2:\\ Input: "v t k g n k o h m p 4"\\ Output: o\\ \\ Example 3:\\ Input: "w y b c v t r i p p 10"\\ Output: w\\ \\ Task:\\ Input: "\{sequence\}"\\ Output:\end{tabular}                                                                                   &  \\ \cmidrule(r){1-2}
Letter & \begin{tabular}[c]{@{}l@{}}You will be given a sequence of space-separated tokens. Your task is to find the target \\ token by following these rules:\\ \\ 1. Start with the very last token in the sequence.\\ 2. Identify the second character of that token.\\ 3. Scan backwards (to the left) to find the token that starts with that exact character; \\ there is only one token that fulfills this condition.\\ 4. This is your target token.\\ \\ Output just the target token and absolutely nothing else. Do not include any explanations\\ or formatting.\\ \\ Example:\\ Input: "a4 b3 c2 d1 fc"\\ Output: c2\\ \\ Input: "w9 t7 l5 g1 n4 l6 u9 k7 m5 p8 gk"\\ Output: k7\\ \\ Input: "q5 r7 x4 t7 f4 k7 q4 u2 u3 r5 ex"\\ Output: x4\\ \\ Task:\\ Input: "\{sequence\}"\\ Output:\end{tabular} &  \\ \cmidrule(r){1-2}
\end{tabular}
\caption{Prompts used to evaluate LLMs on the 1-hop number and letter tasks.}
\label{table:prompts}
\end{table}

\begin{table}[]
\centering
\begin{tabular}{@{}rrrrrrr@{}}
\toprule
\multicolumn{1}{l}{\textbf{Length}} & \multicolumn{2}{l}{\textbf{GPT 5.4}}                                      & \multicolumn{2}{l}{\textbf{GPT 5.5}}                                      & \multicolumn{2}{l}{\textbf{Claude Sonnet 3.7}}                            \\ \midrule
\multicolumn{1}{l}{\textbf{}}       & \multicolumn{1}{l}{\textbf{Number}} & \multicolumn{1}{l}{\textbf{Letter}} & \multicolumn{1}{l}{\textbf{Number}} & \multicolumn{1}{l}{\textbf{Letter}} & \multicolumn{1}{l}{\textbf{Number}} & \multicolumn{1}{l}{\textbf{Letter}} \\
12                                  & 0.98 (0.01)                         & 0.96 (0.01)                         & 0.55 (0.02)                         & 0.87 (0.02)                         & 0.4 (0.02)                          & 0.92 (0.03)                         \\
16                                  & 0.87 (0.01)                         & 0.98 (0.01)                         & 0.56 (0.03)                         & 0.95 (0.01)                         & 0.34 (0.03)                         & 0.89 (0.02)                         \\
24                                  & 0.73 (0.03)                         & 1.0 (0.01)                          & 0.42 (0.03)                         & 0.96 (0.02)                         & 0.22 (0.02)                         & 0.88 (0.03)                         \\
32                                  & 0.45 (0.03)                         & 0.99 (0.0)                          & 0.23 (0.01)                         & 0.96 (0.01)                         & 0.15 (0.02)                         & 0.88 (0.02)                         \\
48                                  & 0.17 (0.03)                         & 0.99 (0.01)                         & 0.21 (0.03)                         & 0.77 (0.01)                         & 0.11 (0.01)                         & 0.8 (0.02)                          \\
64                                  & 0.14 (0.03)                         & 0.98 (0.01)                         & 0.12 (0.02)                         & 0.84 (0.02)                         & 0.08 (0.03)                         & 0.87 (0.02)                         \\
80                                  & 0.1 (0.02)                          & 0.96 (0.01)                         & 0.1 (0.02)                          & 0.75 (0.01)                         & 0.06 (0.01)                         & 0.89 (0.02)                         \\
100                                 & 0.08 (0.02)                         & 0.85 (0.05)                         & 0.07 (0.02)                         & 0.66 (0.04)                         & 0.06 (0.02)                         & 0.86 (0.03)                         \\
200                                 & 0.08 (0.01)                         & 0.52 (0.03)                         & 0.06 (0.01)                         & 0.55 (0.01)                         & 0.05 (0.01)                         & 0.77 (0.02)                         \\
300                                 & 0.06 (0.01)                         & 0.4 (0.06)                          & 0.04 (0.01)                         & 0.5 (0.03)                          & 0.05 (0.01)                         & 0.8 (0.05)                          \\ \bottomrule
\end{tabular}
\caption{Mean accuracy and standard deviation across runs for the LLM generalization experiment.}
\label{table:generalization_llm}
\end{table}

\clearpage
\section{Compute resources}

All GPT-J experiments were run on a local compute server equipped with an AMD EPYC 7282 16-Core Processor and three NVIDIA RTX A5000 GPUs, each with 24 GB of memory. The experimental data were stored in S3 buckets containing 1,000 objects with a total size of approximately 2.0 TB. The total training time for the reported GPT-J experiments was 54 minutes and 59 seconds. Computing the positional and symbolic metrics for each model and checkpoint required 5 hours of GPU time to generate the attention weights and 8.5 hours of CPU time to compute the metrics. Since we trained two models, the total cost of metric computation was 10 hours of GPU time and 17 hours of CPU time.

\begin{table}[h]
\centering
\caption{Compute, storage, and training time used in the GPT-J experiments.}
\label{tab:compute_resources}
\begin{tabular}{ll}
\hline
\textbf{Resource} & \textbf{Specification} \\
\hline
CPU & AMD EPYC 7282 16-Core Processor \\
GPU 0 & NVIDIA GeForce RTX A5000, 24GB \\
GPU 1 & NVIDIA GeForce RTX A5000, 24GB \\
GPU 2 & NVIDIA GeForce RTX A5000, 24GB \\
Storage & S3 buckets, 1k objects, 2.0TB total \\
\hline
Number Task training time & 28 minutes 30 seconds \\
Letter Task training time & 26 minutes 29 seconds \\
Total GPT-J training time & 54 minutes 59 seconds \\
\hline
Attention weights generation (per model) & 5 hours (GPU) \\
Metrics computation & 8.5 hours (CPU) \\
Total metrics (GPU) & 10 hours (GPU) \\
Total metrics (CPU) & 17 hours (CPU) \\
\hline
\end{tabular}
\end{table}

%%%%%%%%%%%%%%%%%%%%%%%%%%%%%%%%%%%%%%%%%%%%%%%%%%%%%%%%%%%%

\end{document}